\documentclass[12pt]{article} % For LaTeX2e
\usepackage{times}
\usepackage{url}
\usepackage{fullpage}
\usepackage{authblk}

\usepackage{graphicx}
\graphicspath{{./image/}}
\usepackage{subfigure}
\usepackage[vlined, ruled]{algorithm2e}
\usepackage{algorithmic}

\newcommand{\citet}[1]{\cite{#1}}
\newcommand{\citep}[1]{\cite{#1}}

\usepackage[pdftex]{hyperref}
\hypersetup{colorlinks=true, citecolor=blue, anchorcolor=blue, linkcolor=blue,
    filecolor=blue, menucolor=blue, urlcolor=black, plainpages=false,
    pdfpagelabels}

\usepackage{enumerate}
\usepackage{amsmath}
\usepackage{amsfonts}
\usepackage{amssymb}
\usepackage{amsthm}
\usepackage{mathtools}

\newcommand{\abs}[1]{|#1|}

\newcommand{\wh}[1]{\widehat{#1}}

\DeclarePairedDelimiter{\norm}{|\!|\!|}{|\!|\!|}
\DeclarePairedDelimiter{\ceil}{\lceil}{\rceil}
\DeclarePairedDelimiter{\floor}{\lfloor}{\rfloor}
\DeclarePairedDelimiter\ip{\langle}{\rangle}

\newcommand{\T}{\top}
\newcommand{\R}{\mathbb{R}}

\def\P{{\mathbb P}}

\newcommand{\dual}{\lucky^*}
\newcommand{\lucky}{s}
\newcommand{\fold}{{\rm fold}}

\newcommand{\vu}{\boldsymbol{u}}
\newcommand{\vv}{\boldsymbol{v}}
\newcommand{\vw}{\boldsymbol{w}}
\newcommand{\vx}{\boldsymbol{x}}

\newcommand{\vz}{\boldsymbol{z}}

\newcommand{\vr}{\boldsymbol{r}}

\newcommand{\vgamma}{\boldsymbol{\gamma}}
\newcommand{\vomega}{\boldsymbol{\omega}}

\newcommand{\mE}{\boldsymbol{E}}
\newcommand{\mG}{\boldsymbol{G}}
\newcommand{\mX}{\boldsymbol{X}}
\newcommand{\mY}{\boldsymbol{Y}}
\newcommand{\mC}{\boldsymbol{C}}
\newcommand{\mU}{\boldsymbol{U}}

\newcommand{\mB}{\boldsymbol{B}}
\newcommand{\mS}{\boldsymbol{S}}
\newcommand{\mM}{\boldsymbol{M}}
\newcommand{\mP}{\boldsymbol{P}}
\newcommand{\mQ}{\boldsymbol{Q}}
\newcommand{\mI}{\boldsymbol{I}}
\newcommand{\mD}{\boldsymbol{D}}
\newcommand{\mW}{\boldsymbol{W}}
\newcommand{\mZ}{\boldsymbol{Z}}
\newcommand{\mSigma}{\boldsymbol{\Sigma}}
\newcommand{\mLambda}{\boldsymbol{\Lambda}}
\newcommand{\mDelta}{\boldsymbol{\Delta}}
\newcommand{\mR}{\boldsymbol{R}}

\newcommand{\tS}{\mathcal{S}}
\newcommand{\tZ}{\mathcal{Z}}
\newcommand{\tD}{\mathcal{D}}

\newcommand{\tX}{\mathcal{X}}
\newcommand{\tW}{\mathcal{W}}
\newcommand{\tY}{\mathcal{Y}}

\newcommand{\tN}{\mathcal{N}}
\newcommand{\tC}{\mathcal{C}}
\newcommand{\tE}{\mathcal{E}}

\def \epsilon{\varepsilon}
\DeclareMathOperator*{\argmin}{arg\,min}

\newcommand{\Id}{\boldsymbol{1}}

\newtheorem{thm}{Theorem}
\newtheorem{lemma}{Lemma}
\newtheorem{defn}{Definition}
\newtheorem{coro}{Corollary}

\title{\Large\bf Interpolating Convex and Non-Convex
Tensor Decompositions via the Subspace Norm}
\author[1]{Qinqing Zheng}
\author[2]{Ryota Tomioka}
\affil[1]{University of Chicago}
\affil[2]{Toyota Technological Institute at Chicago}
\date{}
\begin{document}
\maketitle
\begin{abstract} 
We consider the problem of recovering a low-rank tensor from its noisy
observation. Previous work has shown a recovery guarantee with signal to noise
ratio $O(n^{\ceil{K/2}/2})$ 
for recovering a $K$th order rank
one tensor of size $n\times \cdots \times n$ by recursive unfolding. In this
paper, we first improve this bound to $O(n^{K/4})$ by a much simpler approach,
but with a more careful analysis. 
Then we propose a new norm called the \textit{subspace} norm, which is based on the Kronecker products of factors
obtained by the proposed simple estimator. The imposed Kronecker structure
allows us to show a nearly ideal $O(\sqrt{n}+\sqrt{H^{K-1}})$
bound,
%proposed subspace norm, 
in which the parameter $H$ controls the blend from the non-convex estimator to
mode-wise nuclear norm minimization. Furthermore, we empirically demonstrate
that the subspace norm achieves the nearly ideal denoising performance even with
$H=O(1)$.

\end{abstract} 

\section{Introduction}
Tensor is a natural way to express higher order interactions for a variety
of data and tensor decomposition has been successfully applied to wide
areas ranging from chemometrics, signal processing,
to neuroimaging; see \citet{KolBad09,Moe11} for a survey. Moreover,
recently it has become an active area in the context of learning latent
variable models \cite{AnaGeHsuKakTel14}. 

Many problems related to tensors, such as, finding the rank, or a best
rank-one approaximation of a tensor is known to be NP hard
\cite{Has90,HilLim13}.
 Nevertheless we can address statistical problems,
such as, how well we can recover a low-rank tensor from its randomly
corrupted version (tensor denoising) or from partial observations (tensor
completion). 
Since we can
convert a tensor into a matrix by an operation known as {\em unfolding},
recent work \cite{TomSuzHayKas11,MuHuaWriGol14,RicMon14,JaiOh14} has shown
that we do get nontrivial guarantees by using some norms or
singular value decompositions. More specifically, Richard \& Montanari \citet{RicMon14} has
shown that when a rank-one $K$th order tensor of size
$n\times\cdots\times n$ is corrupted by standard Gaussian noise,
a nontrivial bound can be shown with high probability if the signal to noise ratio
$\beta/\sigma\succsim n^{\ceil{K/2}/2}$  by a method called the
recursive unfolding\footnote{We say $a_n \succsim b_n$ if there is a
constant $C>0$ such that $a_n\geq C\cdot b_n$. }. Note that 
 $\beta/\sigma\succsim \sqrt{n}$ is sufficient for matrices ($K=2$)
and also for tensors if we use the best rank-one approximation (which is known
to be NP hard) as an estimator. On the other hand, Jain \& Oh \citet{JaiOh14} analyzed the tensor
completion problem and proposed an algorithm that requires
$O(n^{3/2}\cdot{\rm polylog}(n))$
samples for $K=3$; while information theoretically we need at least
$\Omega(n)$ samples and the intractable maximum
likelihood estimator would require $O(n\cdot{\rm
polylog}(n))$ samples. Therefore, in both settings, there is a wide gap between 
the ideal estimator and current polynomial time
algorithms. A subtle question that we will address in this paper is
whether we need to unfold the tensor so that the resulting matrix become
as square as possible, which was the reasoning underlying both \citet{MuHuaWriGol14,RicMon14}.

As a parallel development, non-convex estimators based on alternating minimization
or nonlinear optimization \cite{AcaDunKolMor11,SorVanDeL13}
have been widely applied and have performed very well when appropriately
set up. Therefore it would be of fundamental importance to connect the
wisdom of non-convex estimators with the more theoretically motivated
estimators that recently emerged.

In this paper, we explore such a connection by defining a new norm based on
Kronecker products of factors that can be obtained by simple mode-wise
singular value decomposition (SVD) of unfoldings (see notation section below),
also known as
the higher-order singular value decomposition (HOSVD) \cite{DeLDeMVan00a,
DeLDeMVan00b}. 
%Our contributions are two folds. 
We first study the non-asymptotic behavior of the leading
singular vector from the ordinary (rectangular) unfolding $\mX_{(k)}$
and show a nontrivial bound for signal to noise ratio 
 $\beta/\sigma\succsim n^{K/4}$. Thus the
result also applies to odd order tensors confirming a conjecture in
\citet{RicMon14}. Furthermore, this motivates us to use the solution of
mode-wise truncated SVDs to construct a new norm.
We propose the subspace norm, which predicts an unknown low-rank tensor as
a mixture of $K$ low-rank tensors, in which each term takes the form
\begin{align*}
 \fold_k(\mM^{(k)}( \wh{\mP}^{(1)}\otimes \cdots \otimes
 \wh{\mP}^{(k-1)}\otimes  \wh{\mP}^{(k+1)}\otimes \cdots \otimes  \wh{\mP}^{(K)})^\T),
\end{align*}
where $\fold_k$ is the inverse of unfolding $(\cdot)_{(k)}$, $\otimes$
denotes the Kronecker product, and
$\wh{\mP}^{(k)}\in\R^{n\times H}$ is a orthonormal matrix estimated from
the mode-$k$ unfolding of the observed tensor, for $k=1,\ldots,K$; $H$
is a user-defined parameter, and $\mM^{(k)}\in\R^{n\times H^{K-1}}$. Our
theory tells us that with sufficiently high signal-to-noise ratio the
estimated $\wh{\mP}^{(k)}$ spans the true factors.

We highlight our contributions below:
\smallbreak
\noindent 1. We prove that the required signal-to-noise ratio for recovering a
       $K$th order rank one tensor from the ordinary unfolding is
       $O(n^{K/4})$. Our analysis shows a curious two phase behavior: with high probability,
       when $n^{K/4}\precsim \beta/\sigma\precsim n^{K/2}$, the error shows a
       fast decay as $1/\beta^4$; for $\beta/\sigma\succsim n^{K/2}$, the
       error decays slowly as $1/\beta^2$. We confirm this in a
       numerical simulation.\\
\noindent 2. The proposed subspace norm is an interpolation between the intractable
estimators that directly control the rank (e.g.,
HOSVD) and the tractable norm-based
estimators. It becomes equivalent to the latent trace norm 
\citet{TomSuz13} when $H=n$ at the cost of increased
signal-to-noise ratio threshold (see Table \ref{tab:comparison}).\\
\noindent 3. The proposed estimator is more efficient than previously proposed
norm based estimators, because the size of the SVD required in the
algorithm is reduced from $n\times n^{K-1}$ to $n\times H^{K-1}$.\\
\noindent 4. We also empirically demonstrate that the proposed
subspace norm performs nearly optimally for constant order $H$.

\begin{table*}[tb]
 \begin{center}
\caption{Comparison of required signal-to-noise ratio $\beta/\sigma$ of
  different algorithms for recovering a $K$th order rank one tensor of
  size $n\times \cdots\times n$ contaminated by Gaussian noise with
  Standard deviation $\sigma$. See model \eqref{eq:asymm_rank1}.
 The bound for the ordinary unfolding is shown in Corollary
  \ref{coro:phasetransition}. The bound for the subspace norm is shown in
  Theorem \ref{thm:lucky}. The ideal estimator is proven in Appendix \ref{sec:MLE}.}
  \label{tab:comparison}
\smallbreak
  \begin{tabular}[tb]
     % {p{2cm}|p{2cm}|p{2cm}|p{3.2cm}|c}
      { p{2.5cm} | p{2.8cm} | p{2.5cm} | p {3.5cm} | c}
   Overlapped/ Latent nuclear norm\citet{TomSuz13}&
 Recursive unfolding\citet{RicMon14}/ square norm\citet{MuHuaWriGol14} &
 Ordinary unfolding &
 Subspace norm (proposed) & Ideal \\
\hline
$O(n^{(K-1)/2})$ & 
 $O(n^{\ceil{K/2}/2})$ &
 $\boldsymbol{O(n^{K/4})}$ & $\boldsymbol{O(\sqrt{n}+\sqrt{H^{K-1}})}$ &
$\boldsymbol{O(\sqrt{nK\log(K)})}$
  \end{tabular}
 \end{center}
\end{table*}

\subsection*{Notation}
Let $\tX\in\R^{n_1\times n_2\times \cdots\times n_K}$ be a $K$th order tensor. We will often use $n_1=\cdots=n_K=n$ to simplify the
notation but all the results in this paper generalizes to general dimensions.
 The inner product between a pair of tensors is
defined as the inner products of them as vectors; i.e.,
$\ip{\tX,\tW}=\ip{{\rm vec}(\tX),{\rm vec}(\tW)}$.
For $\vu\in\R^{n_1},\vv\in\R^{n_2},\vw\in\R^{n_3}$, $\vu\circ\vv\circ\vw$ denotes the $n_1\times
n_2\times n_3$ {\em rank-one} tensor whose $i,j,k$ entry is $u_iv_jw_k$. The
rank of $\tX$ is the minimum number of rank-one tensors
required to write $\tX$ as a linear combination of them.
 A mode-$k$ fiber of tensor $\tX$
is an $n_k$ dimensional vector that is obtained by fixing all but the
$k$th index of $\tX$. The mode-$k$ 
unfolding $\mX_{(k)}$ of tensor $\tX$ is an $n_k\times \prod_{k'\neq k}n_{k'}$ matrix
constructed by concatenating all the mode-$k$ fibers along columns.
We denote the spectral and Frobenius norms
for matrices by $\|\cdot\|$ and $\|\cdot\|_F$, respectively.

\section{The power of ordinary unfolding}
\label{sec:theory}

\subsection{A perturbation bound for the left singular vector}
We first establish a bound on recovering the left singular vector of a
rank-one $n\times m$ matrix (with $m>n$) perturbed by random Gaussian noise.
%a random Gaussian noise when $m$ is sufficiently larger than $C$.

Consider the following model known as the information plus noise model \cite{BenNad11}:
\begin{equation}
 \label{eq:infonoise}
    \tilde{\mX} = \beta \vu \vv^\T  + \sigma \mE,
\end{equation}
where $\vu$ and $\vv$ are unit vectors, $\beta$ is the signal strength,
$\sigma$ is the noise standard deviation, and
the noise matrix $\mE$ is assumed to be random with entries sampled
i.i.d. from the standard normal distribution. Our goal is to lower-bound the
correlation between $\vu$ and the top left singular vector $\hat{\vu}$
of $\tilde{\mX}$ for signal-to-noise ratio $\beta/\sigma\succsim
(mn)^{1/4}$ with high probability.

A direct application of the classic Wedin perturbation theorem
\cite{Wed72} to the rectangular matrix $\tilde{\mX}$ does
not provide us the desired result. This is because it requires the
signal to noise ratio $\beta/\sigma\geq 2\|\mE\|$. Since the spectral
norm of $\mE$ scales as $O_p(\sqrt{n} + \sqrt{m})$ \cite{Ver10}, this would
mean that we require $\beta/\sigma\succsim m^{1/2}$; i.e., the threshold is
dominated by the number of columns $m$, if $m\geq n$.

Alternatively, we can view $\hat{\vu}$ as the leading eigenvector of 
$\tilde{\mX}\tilde{\mX}^\T$, a square matrix. Our key insight is that
we can decompose $\tilde{\mX}\tilde{\mX}^\T$ as follows:
\begin{align*}
        \tilde{\mX}\tilde{\mX}^\T  &= (\beta^2 \vu \vu^\T +m\sigma^2\mI)  + (\sigma^2 \mE \mE^\T -m\sigma^2\mI)+
        \beta \sigma (\vu \vv^\T \mE^\T + \mE \vv \vu^\T).
\end{align*}
Note that $\vu$ is the leading eigenvector of the first term because
adding an identity matrix does not change the eigenvectors. Moreover, we
notice that there are two noise terms: the first term is a centered
Wishart matrix and it is independent of the signal $\beta$; the second
term is Gaussian distributed and depends on the signal $\beta$.

This implies a two-phase behavior corresponding to either 
the Wishart or the Gaussian noise term being dominant,
depending on the value of $\beta$. Interestingly, we get a different
speed of convergence for each of these phases as we show in the next
theorem (the proof is given in Appendix \ref{sec:proof-thm-leftsv}).

\begin{thm}
\label{thm:leftsv}
There exists a constant $C$ such that with probability at least $1-4e^{-n}$,
if $ m / n \geq C $,
 \[ \abs{\ip{\hat{\vu}, \vu }} \geq 
     \begin{dcases}
         1 - \frac{C nm }{(\beta/\sigma)^4}, & \text{if}\;\; 
            \sqrt{m} > \frac{\beta}{\sigma} \geq (Cnm)^{\frac{1}{4}}, \\
         1 - \frac{C n}{(\beta/\sigma)^2}, & \text{if}\;\;
            \frac{\beta}{\sigma} \geq \sqrt{m},\\
     \end{dcases}
\]
otherwise, $\abs{\ip{\hat{\vu}, \vu }} \geq 1 - \frac{C
n}{(\beta/\sigma)^2}$ if $\beta/\sigma \geq \sqrt{Cn}$.
\end{thm}
%We prove a more general version
%of the theorem that allows the signal part to be rank $R$ in Appendix \ref{sec:leftsv_rankR}. 

In other words, if $\tilde{\mX}$ has sufficiently many more columns than rows,
as the signal to noise ratio $\beta / \sigma$ increases,
$\hat{\vu}$ first converges to $\vu$ as $1/\beta^4$, and then as $1/\beta^2$.
Figure \ref{fig:leftsv} illustrates these results.
We randomly generate a rank-one $100 \times 10000$ matrix perturbed by Gaussian
noise, and measure the distance between $\hat{\vu}$ and $\vu$.
The phase transition happens at $\beta/\sigma = (nm)^{1/4}$, 
and there are two regimes of different convergence rates as Theorem \ref{thm:leftsv} predicts.

%\begin{figure}[tb]
% \begin{center}
% \end{center}
%\end{figure}
\begin{figure}[tb]
 \begin{center}
\subfigure[
Synthetic experiment showing phase transition at
  $\beta/\sigma = (nm)^{1/4}$ and regimes with different rates of convergence.
  See Theorem \ref{thm:leftsv}.
%  As $\beta/\sigma$ grows, the distance between $\hat{\vu}$ and $\vu$
%  decreases as $1/\beta^4$ between $(nm)^{1/4}$ and $\sqrt{m}$, and as
%  $1/\beta^2$ after $\sqrt{m}$.
]{
  \includegraphics[width=.45\columnwidth]{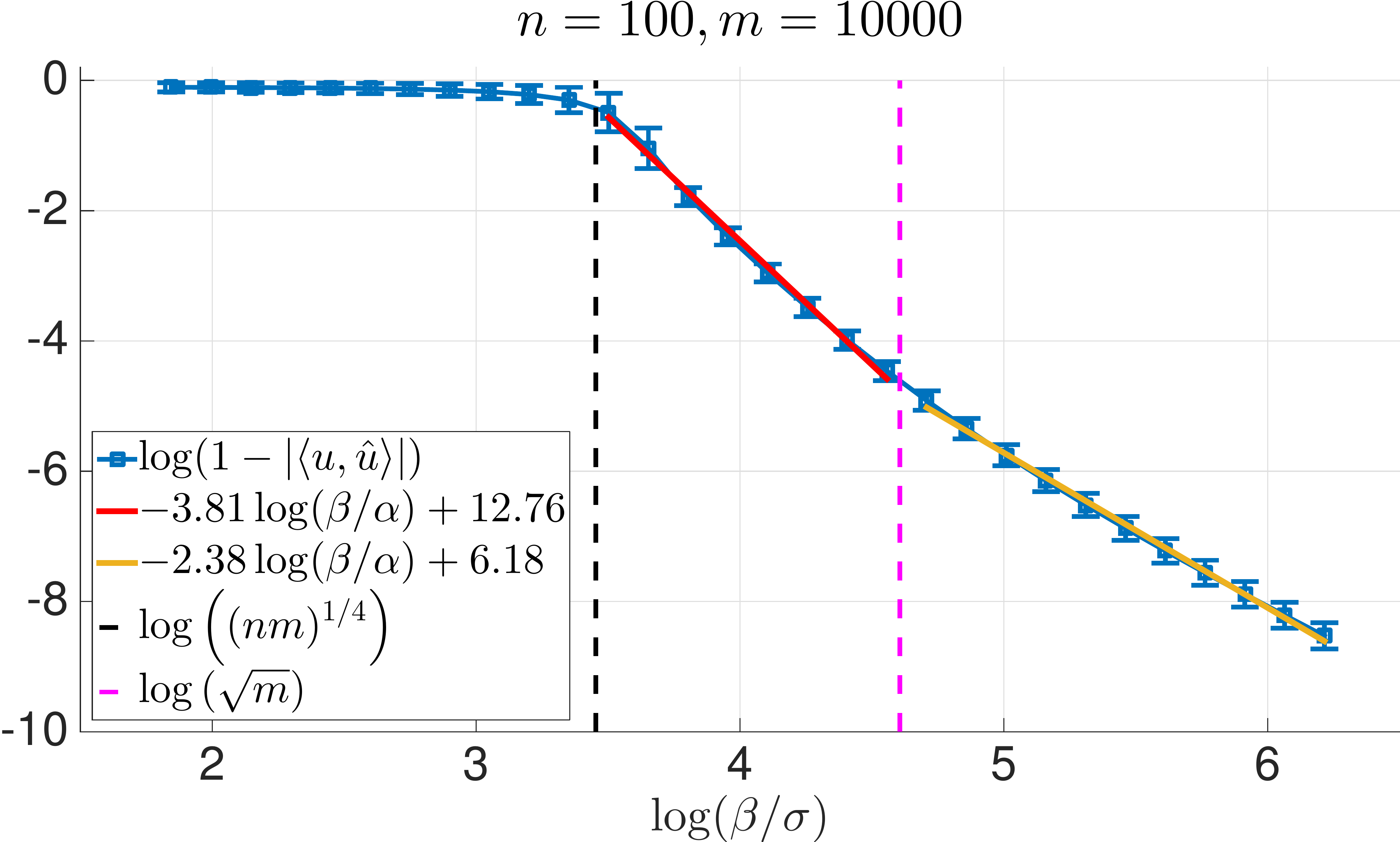}
  \label{fig:leftsv}
}\hspace*{5mm}\subfigure[
Synthetic experiment showing phase transition at
  $\beta=\sigma(\prod_{k}n_k)^{1/4}$ for odd order tensors. 
See Corollary \ref{coro:phasetransition}.
% The observed
%  tensor $\tY$ is generated as $\tY=\tX^{\ast}+\sigma\tE$ with
%  the signal tensor $\tX^{\ast}$ defined in \eqref{eq:rank-r-cp}. The
%  inner products go down from one to zero symmetrically around
%  $\sigma=\beta/(\prod_{k}n_k)^{1/4}$ as predicted by
]{\includegraphics[width=.45\columnwidth]{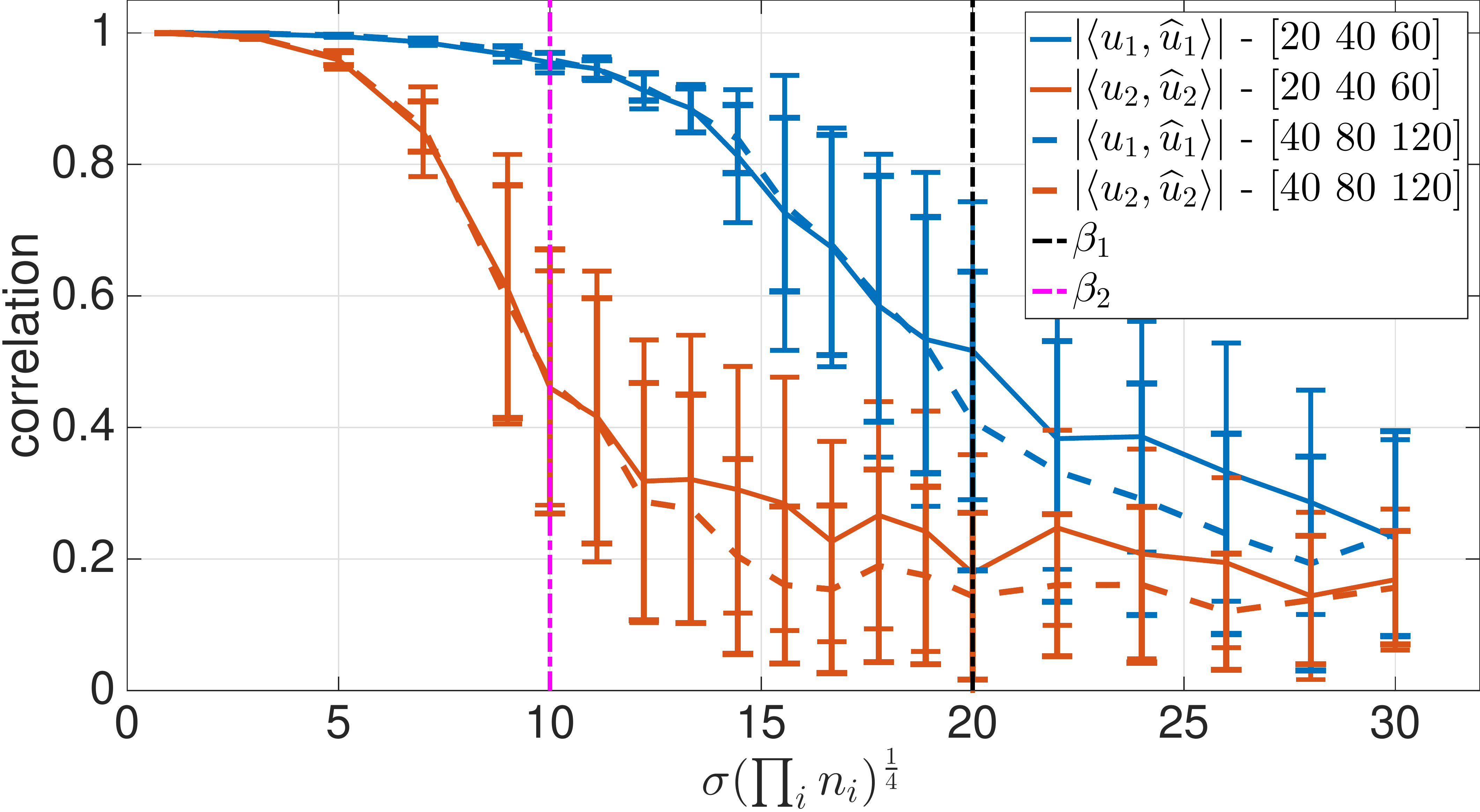}
  \label{fig:correlation}}
 \end{center}
\vspace*{-2mm}
 \caption{Numerical demonstration of Theorem \ref{thm:leftsv} and
 Corollary \ref{coro:phasetransition}.}
\end{figure}

\subsection{Tensor Unfolding}
Now let's apply the above result to the tensor version of information plus noise model
studied by \citet{RicMon14}. We consider a rank one $n\times 
\cdots \times n$ tensor (signal) contaminated by Gaussian noise as follows:
\begin{equation}
    \label{eq:asymm_rank1}
    \tY = \tX^{\ast} +\sigma\tE = \beta \vu^{(1)} \circ \cdots \circ \vu^{(K)} + \sigma\tE, 
\end{equation}
where factors $\vu^{(k)}\in\R^{n}$, $k=1,\ldots,K$, are unit vectors, which are not necessarily
identical, and the entries of $\tE\in\R^{n\times\cdots\times n}$ are i.i.d samples from the normal distribution
$\mathcal{N}(0,1)$. Note that this is slightly more general (and easier
to analyze) than the symmetric setting studied by \citet{RicMon14}.

%We could unfold the tensor along the first mode into a matrix $Y_{(1)} = \beta u^{(1)} (u^{(K)} \otimes \ldots
%\otimes u^{(2)})^\T + \EE_{(1)}$. One can see that the columns of $Y_{(1)}$ are scaled $u^{(1)}$ 
%perturbed by some noise. More precisely, let $z = u^{(K)} \otimes \ldots \otimes u^{(1)}$, and the $k$-the column of $Y_{(1)}$ reads
%\[ Y_{(1), k} = \beta z_k u^{(1)}  + \sigma^2 \epsilon_k, \]
%where $ \epsilon_k \sim \N(0, I)$.
%
%For notation simplicity, we drop the subscript for the mode in the sequel. Note since $\norm{z}_2=1$, the second moment of $Y_k$ is
%\[ \E(YY^\T) = \E( \beta^2 u z^\T z u^\T + \sigma^2 \EE \EE^\T) = \beta^2 uu^\T + \sigma^2 I .\]

Several estimators for recovering $\tX^*$ from its noisy
version $\tY$ have been proposed (see Table \ref{tab:comparison}). Both
the overlapped nuclear norm and latent nuclear norm discussed in \cite{TomSuz13}
achives the relative performance guarantee
\begin{align}
\label{eq:denoise-overlap}
 \norm{\hat{\tX}-\tX^{\ast}}_F/\beta\leq O_p\left(\sigma\sqrt{n^{K-1}}/\beta\right),
\end{align}
% $\|\mX\|_{\ast}=\sum_{j=1}^{r}\sigma_j(\mX)$ is the nuclear norm
%(also known as the trace norm), $r$ is the rank of $\mX$, and
% $\mX_{(k)}$ denotes the mode-$k$ unfolding of $\tX$.
where $\hat{\tX}$ is the estimator. This bound
implies that if we want to
obtain relative error smaller than $\epsilon$, we need the signal to
noise ratio $\beta/\sigma$ to scale as $\beta/\sigma\succsim
\sqrt{n^{K-1}}/\epsilon$. 

Mu et al. \citet{MuHuaWriGol14} proposed the square norm, 
defined as the nuclear norm of the matrix obtained by grouping the first
$\floor{K/2}$ indices along the rows and the last $\ceil{K/2}$ indices
along the columns. This norm improves the right
hand side of inequality \eqref{eq:denoise-overlap} to
$O_p(\sigma\sqrt{n^{\ceil{K/2}}}/\beta)$, which translates to
requiring $\beta/\sigma\succsim \sqrt{n^{\lceil K/2\rceil}}/\epsilon$ for
obtaining relative error $\epsilon$. The intuition here is the more
square the unfolding is the better the bound becomes. However, there is
no improvement for $K=3$.

Richard and Montanari \citet{RicMon14} studied the (symmetric version of) model
\eqref{eq:asymm_rank1} and proved that a 
recursive unfolding algorithm achieves the 
factor recovery error
 ${\rm dist}(\hat{\vu}^{(k)},\vu^{(k)})=\epsilon$ with
$\beta/\sigma\succsim \sqrt{n^{\ceil{K/2}}}/\epsilon$ with high probability, where
 ${\rm dist}(\vu,\vu'):=\min(\|\vu-\vu'\|,\|\vu+\vu'\|)$.
They also showed that the randomly initialized tensor
power method \cite{DeLDeMVan00b,KolMay11,AnaGeHsuKakTel14}
can achieve the same error $\epsilon$ with slightly worse threshold
$\beta/\sigma\succsim \max(\sqrt{n}/\epsilon^2,n^{K/2})\sqrt{K\log K}$
also with high probability.

The reasoning underlying both \citet{MuHuaWriGol14} and \citet{RicMon14}
is that square unfolding is better. However, if we take the (ordinary) mode-$k$
unfolding 
\begin{align}
    \label{eq:rank1_tensor_unfold}
 \mY_{(k)}=\beta \vu^{(k)}\bigl(&\vu^{(k-1)} \otimes \cdots
\otimes \vu^{ (1)}\otimes \vu^{(K)}\otimes \cdots \otimes \vu^{(k+1)}\bigr)^\T + \sigma \mE_{(k)},
\end{align}
we can see  \eqref{eq:rank1_tensor_unfold} as 
an instance of information plus noise model \eqref{eq:infonoise} where
$m/n=n^{K-2}$. Thus the ordinary unfolding satisfies the condition of
Theorem \ref{thm:leftsv} for $n$ or $K$ large enough.

\begin{coro}
\label{coro:phasetransition}
Consider a $K(\geq 3)$th order rank one tensor contaminated by Gaussian noise as
 in \eqref{eq:asymm_rank1}. There exists a constant $C$ such that if $n^{K-2} \geq C$, with probability at least $1 - 4Ke^{-n}$, we have
\[  
 {\rm dist}^2(\hat{\vu}^{(k)},\vu^{(k)}) \leq 
    \begin{dcases}
        \frac{ 2C n^K }{(\beta/\sigma)^4}, & \text{if} \;\; n^{\frac{K-1}{2}} >
        \beta/\sigma \geq C^{\frac{1}{4}}n^\frac{K}{4},\\
        \frac{ 2C n }{(\beta/\sigma)^2}, & \text{if} \;\; \beta/\sigma
        \geq n^{\frac{K-1}{2}},
    \end{dcases}
\qquad\text{for $k=1,\ldots, K$,}
\]
where $\hat{\vu}^{(k)}$ is the leading left singular vector of the
 rectangular unfolding $\mY_{(k)}$.
\end{coro}

This proves that as conjectured by \citet{RicMon14}, the threshold
$\beta/\sigma \succsim n^{K/4}$ applies not only to the even order case but also
to the odd order case. Note that Hopkins et al. \cite{HopShiSte15} have
shown a similar result without the sharp rate of convergence.
The above corollary easily extends to more general $n_1\times\cdots\times
n_K$ tensor by replacing the conditions by $ \sqrt{\prod_{\ell\neq k}n_\ell} > \beta/\sigma\geq
(C\prod_{k=1}^{K}n_k)^{1/4}$ and $ \beta / \sigma \geq
\sqrt{\prod_{\ell\neq k}n_\ell}$. The result also holds when $\tX^\ast$ has rank
higher than 1; see Appendix \ref{sec:leftsv_rankR}.

We demonstrate this result in Figure \ref{fig:correlation}. 
The models behind the experiment are slightly more general ones in which
$[n_1,n_2,n_3]=[20,40,60]$ or $[40,80,120]$
and the signal $\tX^{\ast}$ is rank two with $\beta_1=20$ and
$\beta_2=10$. The plot shows the inner products
$\ip{\vu_1^{(1)},\hat{\vu}_1^{(1)}}$ and
$\ip{\vu_2^{(1)},\hat{\vu}_2^{(1)}}$ as a measure of the quality of
estimating the two mode-1 factors. The horizontal axis is the normalized
noise standard deviation $\sigma(\prod_{k=1}^{K}n_k)^{1/4}$. We can clearly see that
the inner product decays symmetrically around $\beta_1$ and
$\beta_2$ as predicted by Corollary \ref{coro:phasetransition} for both
tensors.

\section{Subspace norm for tensors}
\label{sec:lucky}
%In this section, we propose the \textit{subspace norm} for low rank tensor
%decomposition. 

 Suppose the true tensor $\tX^\ast\in\R^{n\times \cdots\times n}$ admits
 a minimum Tucker decomposition \cite{Tuc66} of rank $(R,\ldots,R)$:
\begin{align}
\label{eq:rank-r-tucker}
 \tX^* = \textstyle\sum_{i_1=1}^R\cdots\textstyle\sum_{i_K=1}^{R} \beta_{i_1i_2\ldots i_K} \vu^{(1)}_{i_1} \circ \cdots \circ \vu^{(K)}_{i_K}.
\end{align}
If the core tensor $\tC=(\beta_{i_1\ldots i_K})\in\R^{R\times \cdots
\times R}$ is {\em superdiagonal},
the above decomposition reduces to the canonical polyadic (CP)
decomposition~\cite{Hit27,KolBad09}.
The mode-$k$ unfolding of the true tensor $\tX^\ast$ can be written as follows:
\begin{align}
\label{eq:mode-k-unfolding}
 \mX^*_{(k)}
 = \mU^{(k)} \mC_{(k)} \left( \mU^{(1)}  \otimes \cdots
 \otimes \mU^{(k-1)}\otimes \mU^{(k+1)}\otimes \cdots\otimes \mU^{(K)} \right)^\T,
\end{align}
where $\mC_{(k)}$ is the mode-$k$ unfolding of the core tensor $\tC$;
$\mU^{(k)}$ is a $n \times R$ matrix $\mU^{(k)}=[\vu^{(k)}_1, \ldots,
\vu^{(k)}_R]$ for $k=1,\ldots,K$. Note that $\mU^{(k)}$ is not
necessarily orthogonal.

Let $\mX^\ast_{(k)}=\mP^{(k)}\mLambda^{(k)}{\mQ^{(k)}}^\top$
be the SVD of $\mX^\ast_{(k)}$.
We will observe that
\begin{align}
\label{eq:span-by-kron}
\mQ^{(k)} \in \text{Span}\left(
\mP^{(1)}\otimes \cdots \otimes \mP^{(k-1)}\otimes
 \mP^{(k+1)}\otimes \cdots\otimes \mP^{(K)}
\right)
\end{align}
because of \eqref{eq:mode-k-unfolding} and $\mU^{(k)}\in\text{Span}(\mP^{(k)})$.

Corollary \ref{coro:phasetransition} shows that the left
singular vectors $\mP^{(k)}$ can be recovered
 under mild conditions; thus
the span of the right singular vectors can also be recovered.
 Inspired by this, we define a norm that models a tensor
$\tX$ as a mixture of tensors $\tZ^{(1)}, \ldots,
\tZ^{(K)}$. We require that the mode-$k$ unfolding of $\tZ^{(k)}$, i.e.
$\mZ^{(k)}_{(k)}$, has a low rank factorization 
$ \mZ^{(k)}_{(k)} = \mM^{(k)} {\mS^{(k)}}^\T,$
where $\mM^{(k)}\in\R^{n\times H^{K-1}}$ is a variable, and $\mS^{(k)}\in\R^{n^{K-1}\times
 H^{K-1}}$ is a fixed arbitrary orthonormal basis of some
 subspace, which we choose later to have the Kronecker structure in
 \eqref{eq:span-by-kron}.
% Since it has higher
% dimensionality than the truth, we penalize the trace norm of the left factor
%$\mM^{(k)}$.

In the following, we define the subspace norm, suggest an approach to construct
the right factor $\mS^{(k)}$, and prove the denoising bound in the end.

\subsection{The subspace norm}
Consider a $K$th order tensor of size $n\times \cdots n$.

\begin{defn}
Let $\mS^{(1)}, \ldots, \mS^{(K)}$ be matrices such that $\mS^{(k)} \in
\R^{n^{K-1}\times H^{K-1}}$ with $H \leq n$.
The subspace norm for a $K$th order tensor $\tX$ associated with
%$\mS^{(1)}, \ldots \mS^{(K)}$ 
$\{ \mS^{(k)} \}_{k=1}^K$ 
is defined as
\[
    \norm{\tX}_\lucky := \begin{cases} 
    \inf_{ \{\mM^{(k)}\}^K_{k=1} }
    \sum_{k=1}^K \| \mM^{(k)} \|_*,  &\text{if} \; \tX \in \text{Span}(\{\mS^{(k)}\}^K_{k=1}), \\
    +\infty, & \text{otherwise},
    \end{cases}
\]
where $\|\cdot\|_\ast$ is the nuclear norm, and $\text{Span}(
 \{\mS^{(k)}\}^K_{k=1}) := \big\{ \tX\in\R^{n\times \cdots\times n}: \exists \mM^{(1)}, \ldots, \mM^{(K)},  \tX = \sum_{k=1}^K \fold_k(\mM^{(k)}{\mS^{(k)}}^\T)  \big \}$.
\end{defn}

In the next lemma (proven in Appendix \ref{sec:proof-dualnorm}), we show
the dual norm of the subspace norm has a simple appealing
form. As we see in Theorem~\ref{thm:lucky}, it avoids the $O(\sqrt{n^{K-1}})$
scaling (see the first column of Table~\ref{tab:comparison}) by restricting the influence
of the noise term in the subspace defined by $\mS^{(1)},\ldots,\mS^{(K)}$.
\begin{lemma}
\label{lem:dualnorm}
 The dual norm of $\norm{\cdot}_\lucky$ is a semi-norm 
\begin{align*}
\norm{\tX}_{\dual} = \max_{k=1,\ldots,K} \| \mX_{(k)} \mS^{(k)}\|,
\end{align*}
where $\|\cdot\|$ is the spectral norm.
\end{lemma}
\subsection{Choosing the subspace}
A natural question that arises is how to choose the matrices $\mS^{(1)}, \ldots,
\mS^{(k)}$.

\begin{lemma}
\label{lemma:span}
Let the $\mX^*_{(k)} = \mP^{(k)} \mLambda^{(k)} \mQ^{(k)}$ be the SVD of
$\mX^*_{(k)}$, where $\mP^{(k)}$ is $n \times R$ and $\mQ^{(k)}$ is
$n^{K-1} \times R$.
Assume that $R \leq n$ and $\mU^{(k)}$ has full column rank. It holds that for all $k$, 
\begin{enumerate}[i)]
\item $\mU^{(k)} \in \text{Span}(\mP^{(k)})$,
\item $\mQ^{(k)} \in 
\text{Span}\left(
\mP^{(1)}\otimes \cdots \otimes \mP^{(k-1)}\otimes
 \mP^{(k+1)}\otimes \cdots\otimes \mP^{(K)}
\right)$.
\end{enumerate}
\end{lemma}
\begin{proof}
We prove the lemma in Appendix \ref{sec:proof-lemma-span}.
\end{proof}

Corollary \ref{coro:phasetransition} shows that when the signal to
noise ratio is high enough, we can recover $\mP^{(k)}$  with high
probability. Hence we suggest the following three-step
approach for tensor denoising:
\begin{enumerate}[(i)]
   \item For each $k$, unfold the observation tensor in mode $k$ and compute the top
   $H$ left singular vectors. Concatenate these vectors to
   obtain a $n \times H$ matrix $\wh{\mP}^{(k)}$.

   \item Construct $\mS^{(k)}$ as $\mS^{(k)}=\wh{\mP}^{(1)} \otimes
	 \cdots \otimes \wh{\mP}^{(k-1)} \otimes \wh{\mP}^{(k+1)}\otimes \cdots \otimes
     \wh{\mP}^{(K)}$.% The size of $\mS^{(k)}$ is $n^{K-1} \times H^{K-1}$.
   \item Solve the subspace norm regularized minimization problem
   \begin{equation}
   \label{eq:lucky}
   \min_{\tX}\quad \frac{1}{2} \norm{\tY - \tX}^2_F + \lambda \norm{\tX}_\lucky,
   \end{equation}
   where the subspace norm is associated with the above defined $\{\mS^{(k)}\}^K_{k=1}$.
\end{enumerate}
See Appendix \ref{sec:optimization} for details.

\subsection{Analysis}
Let $\tY\in\R^{n\times\cdots\times n}$ be a tensor
corrupted by Gaussian noise with standard deviation $\sigma$
as follows:
\begin{align}
\label{eq:noisemodel}
 \tY=\tX^{\ast}+\sigma\tE.
\end{align}
We define a slightly modified estimator $\hat{\tX}$ as follows:
\begin{align}
\label{eq:estimator}
 \hat{\tX}=\argmin_{\tX,\{\mM^{(k)}\}_{k=1}^{K}}\!\!\!\Bigl\{
&\frac{1}{2}\norm{\tY-\tX}_F^2+\lambda\norm{\tX}_\lucky:\;
\tX=\sum_{k=1}^{K}\fold_k\left(\mM^{(k)}{\mS^{(k)}}^\T\right),\{\mM^{(k)}\}_{k=1}^{K}\in\mathcal{M}(\rho)\Bigr\}
\end{align}
where $\mathcal{M}(\rho)$ is a restriction of the set of matrices
$\mM^{(k)}\in\R^{n \times H^{K-1}}$, $k=1,\ldots,K$ defined as follows:
\begin{align*}
 \mathcal{M}(\rho):=\Bigl\{
&\{\mM^{(k)}\}_{k=1}^{K}:
\|\fold_k(\mM^{(k)})_{(\ell)}\|\leq
 \frac{\rho}{K}(\sqrt{n}+\sqrt{H^{K-1}}),\forall k\neq \ell
\Bigr\}.
\end{align*}
This restriction makes sure
that $\mM^{(k)}$, $k=1,\ldots,K$, are incoherent, i.e., each $\mM^{(k)}$
has a spectral norm that is as low
as a random matrix when unfolded at a different mode $\ell$. Similar
assumptions were used in low-rank plus sparse matrix decomposition
\cite{AgaNegWai12,HsuKakZha11} and for the denoising bound for the
latent nuclear norm \cite{TomSuz13}.

Then we have the following statement (we prove this in Appendix \ref{sec:proof-thm-lucky}).

\begin{thm}
\label{thm:lucky}
 Let $\tX_p$ be any tensor that can be expressed as
\begin{align*}
 \tX_p = \sum_{k=1}^{K}\fold_k\left(\mM_p^{(k)}{\mS^{(k)}}^\T\right),
\end{align*}
which satisfies the above incoherence condition $\{\mM_p^{(k)}\}_{k=1}^{K}\in
 \mathcal{M}(\rho)$ and let $r_k$ be the rank of $\mM_p^{(k)}$ for
 $k=1,\ldots,K$. In addition, we assume that each $\mS^{(k)}$ is
 constructed as $\mS^{(k)}=\wh{\mP}^{(k-1)}\otimes \cdots\otimes
 \wh{\mP}^{(k+1)}$ with $(\wh{\mP}^{(k)})^\T\wh{\mP}^{(k)}=\mI_{H}$. 
Then there are universal constants $c_0$ and $c_1$
 such that any solution $\hat{\tX}$ of the minimization problem
 \eqref{eq:estimator} with
 $\lambda=\norm{\tX_p-\tX^{\ast}}_{\dual}+ c_0\sigma\left(\sqrt{n}+\sqrt{H^{K-1}}+\sqrt{2\log(K/\delta})\right)$
 satisfies the following bound
\begin{align*}
 \norm{\hat{\tX}-\tX^{\ast}}_F
\leq  \norm{\tX_p-\tX^{\ast}}_F+c_1\lambda\sqrt{\sum\nolimits_{k=1}^{K}r_k},
\end{align*}
with probability at least $1-\delta$.
\end{thm}

Note that the right-hand side of the bound consists of two terms. The
first term is the approximation error. This term will be zero if
$\tX^{\ast}$ lies in $\text{Span}( \{\mS^{(k)}\}^K_{k=1})$. This is the case, if we
choose $\mS^{(k)}=\mI_{n^{K-1}}$ as in the latent nuclear norm, or if
the condition of Corollary \ref{coro:phasetransition} is satisfied for
the smallest $\beta_R$ when we use the Kronecker product
construction we proposed. Note that the
regularization constant $\lambda$ should also scale with the dual
subspace norm of the residual $\tX_p-\tX^{\ast}$.

The second term is the estimation error with respect to $\tX_p$. If we
take $\tX_p$ to be the orthogonal projection of $\tX^\ast$ to the 
$\text{Span}( \{\mS^{(k)}\}^K_{k=1})$, we can ignore the contribution of
the residual to $\lambda$, because
$(\tX_p-\tX^{\ast})_{(k)}\mS^{(k)}=0$. Then the estimation error
scales mildly with the dimensions $n$, $H^{K-1}$ and with the sum of the ranks.
Note that if we take $\mS^{(k)}=\mI_{n^{K-1}}$, we have $H^{K-1}=n^{K-1}$, and we
recover the guarantee \eqref{eq:denoise-overlap} .

\section{Experiments}
\label{sec:exp}
In this section, we conduct tensor denoising experiments on synthetic
and real datasets, to numerically confirm our analysis in previous sections. 

\subsection{Synthetic data}
% in my code, beta1 = 10, beta2 = 20. here i will flip the order: beta1 = 20, beta2 = 10
We randomly generated the true rank two tensor $\tX^*$ of size $20 \times 30
\times 40$ with singular values $\beta_1 = 20$ and $\beta_2 = 10$. The
true factors are generated as random matrices with orthonormal columns. 
The observation tensor $\tY$ is then generated by adding Gaussian noise
with standard deviation $\sigma$ to $\tX^*$.
%For $\sigma$, we chose 20 values linearly spaced between
%$0.1 \beta_2/(\prod_i n_i)^{1/4}$ and $
%1.2\beta_1/(\prod_i n_i)^{1/4}$.

Our approach is compared to the CP decomposition, the overlapped approach, and
the latent approach. The CP decomposition is computed by the tensorlab \cite{tensorlab} with 20
random initializations. We assume CP knows the true rank is 2. For the
subspace norm, we use Algorithm \ref{alg:admm} described in Section
\ref{sec:lucky}. We also select the top 2 singular
vectors when constructing $\wh{\mU}^{(k)}$'s. We computed the solutions for
20 values of regularization parameter $\lambda$ logarithmically spaced between 1
and 100. For the overlapped and the latent norm, we use ADMM described in
\cite{TomSuzHayKas11}; we also computed 20 solutions with the same $\lambda$'s used for
the subspace norm.

We measure the performance in the relative error defined as $\norm{\wh{\tX} -
\tX^*}_F/\norm{\tX^*}_F$. We report the minimum error obtained by choosing
the optimal regularization parameter or the optimal initialization.
Although the regularization parameter could be selected by leaving out some entries
and measuring the error on these entries, we will not go into tensor
completion here for the sake of simplicity.
%as well as the error for three
%representative values of $\lambda$ for our method. 
%, to verify that the phase
%transitions occur in the relative error as well.

Figure \ref{fig:exp} (a) and (b) show the result of this experiment. The left panel shows
the relative error for 3 representative values of $\lambda$ for the
subspace norm. 
The black dash-dotted line shows the minimum error across all the $\lambda$'s.
The magenta dashed line shows the error corresponding to the theoretically motivated choice 
$\lambda=\sigma(\max_k(\sqrt{n_k}+\sqrt{H^{K-1}})+\sqrt{2\log(K)})$ for each $\sigma$.
The two vertical lines are thresholds of $\sigma$ from Corollary
\ref{coro:phasetransition} corresponding to $\beta_1$ and $\beta_2$, namely, $\beta_1/(\prod_k n_k)^{1/4}$ and $\beta_2/(\prod_k
n_k)^{1/4}$. It confirms that there is a rather sharp increase in the
error around the theoretically predicted places (see also Figure \ref{fig:correlation}). We can also see
that the optimal $\lambda$ should grow linearly with $\sigma$. 
For large $\sigma$ (small SNR), the best relative
error is $1$ since the optimal choice of the regularization parameter
$\lambda$ leads to predicting with $\wh{\tX}=0$.

Figure \ref{fig:exp} (b) compares the performance of the subspace norm to other
approaches. For each method the smallest error corresponding to the
optimal choice of the regularization parameter $\lambda$ is shown.
In addition, to place the numbers in context, we plot the line corresponding to
\begin{align}
\label{eq:optimistic}
\text{Relative error} =  
    \frac{\sqrt{R \sum_k n_k \log(K)}}{\norm{\tX^{\ast}}_F}\cdot\sigma,
\end{align}
which we call ``optimistic''. This can be motivated from considering the
(non-tractable) maximum likelihood estimator for CP decomposition (see  Appendix \ref{sec:MLE}).

Clearly, the error of CP, the subspace norm, and ``optimistic'' grows at the same rate, much
slower than overlap and latent. The error of CP increases beyond $1$, as no
regularization is imposed (see Appendix \ref{sec:moreexp} for more experiments).
%
% The point when the optimistic error reaches 1
%also gives us a critical standard derivation
%
% $\sigma_c =
%\frac{\norm{\tX^{\ast}}_F}{\sqrt{R \sum_k n_k \log (K)}}$.
%
We can see that 
both CP and the subspace norm are behaving near optimally in this setting,
although such behavior is guaranteed for the subspace norm whereas it is hard
to give any such guarantee for the CP decomposition based on nonlinear optimization.

\begin{figure}[t]
\begin{center}
\subfigure[]{\includegraphics[width=.3\columnwidth]{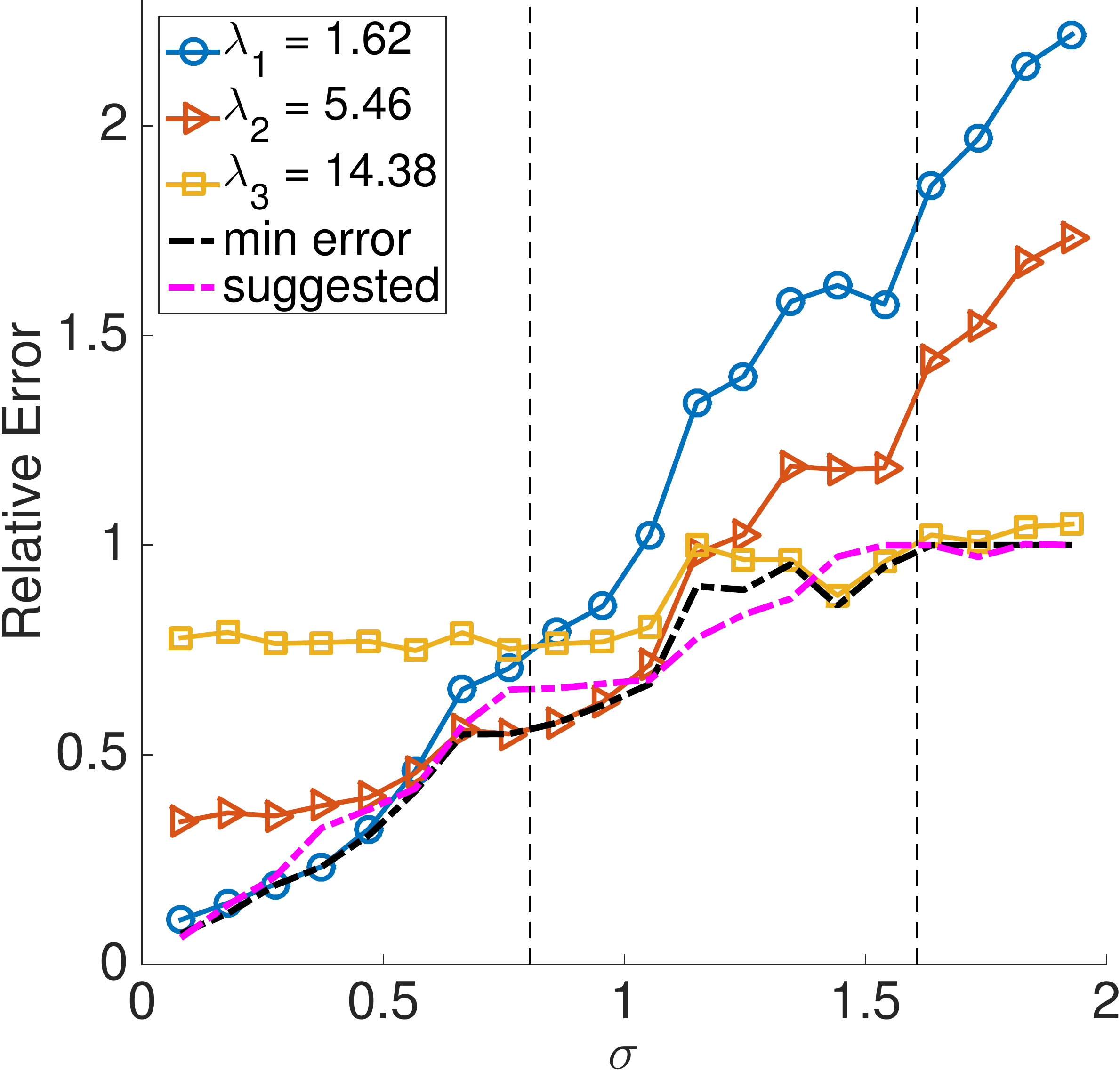}
\label{fig:syn-subspace}}
\subfigure[]{\includegraphics[width=0.3\columnwidth]{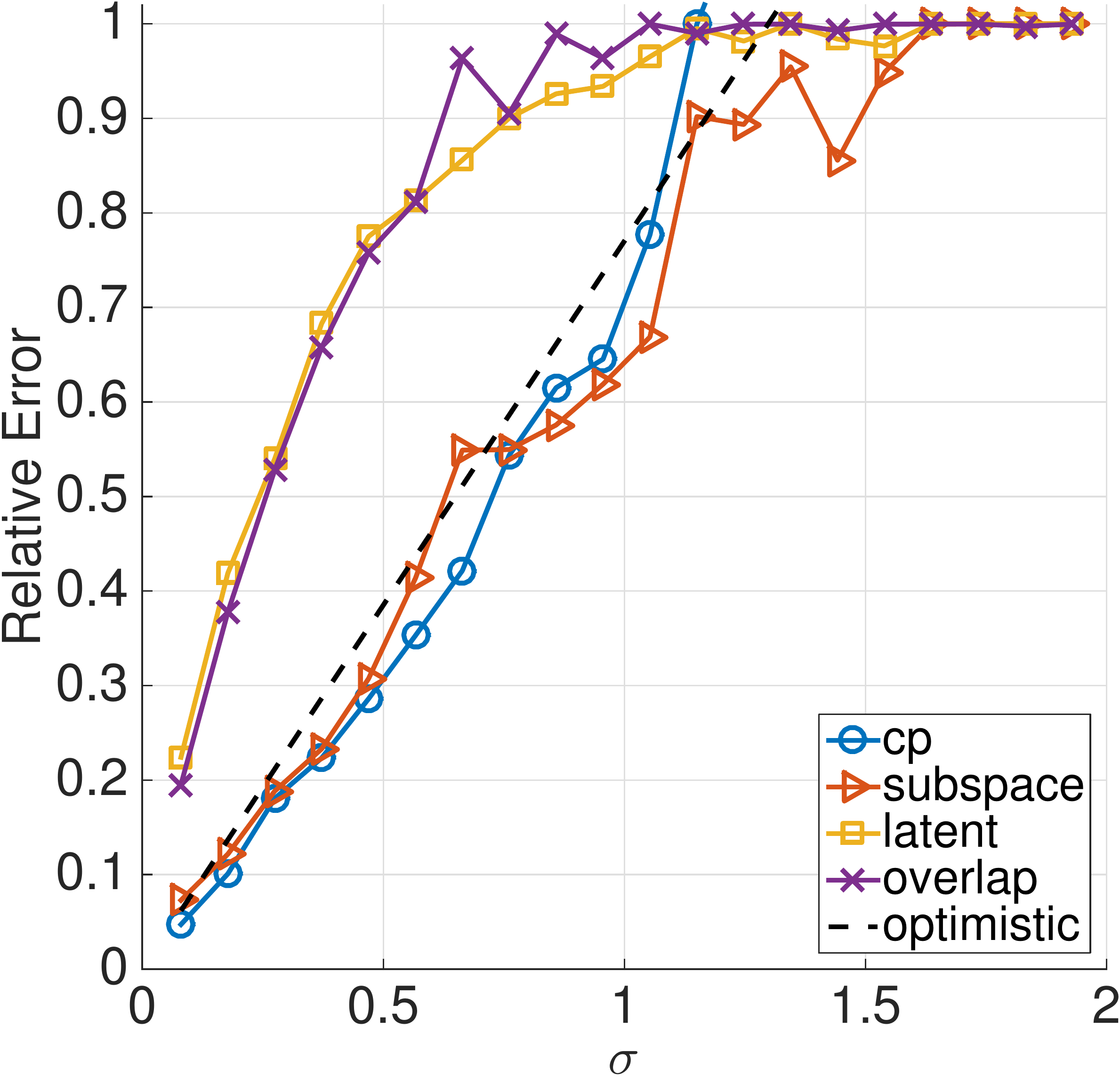}
\label{fig:syn-compare} }
\subfigure[]{\includegraphics[width=.3\columnwidth]{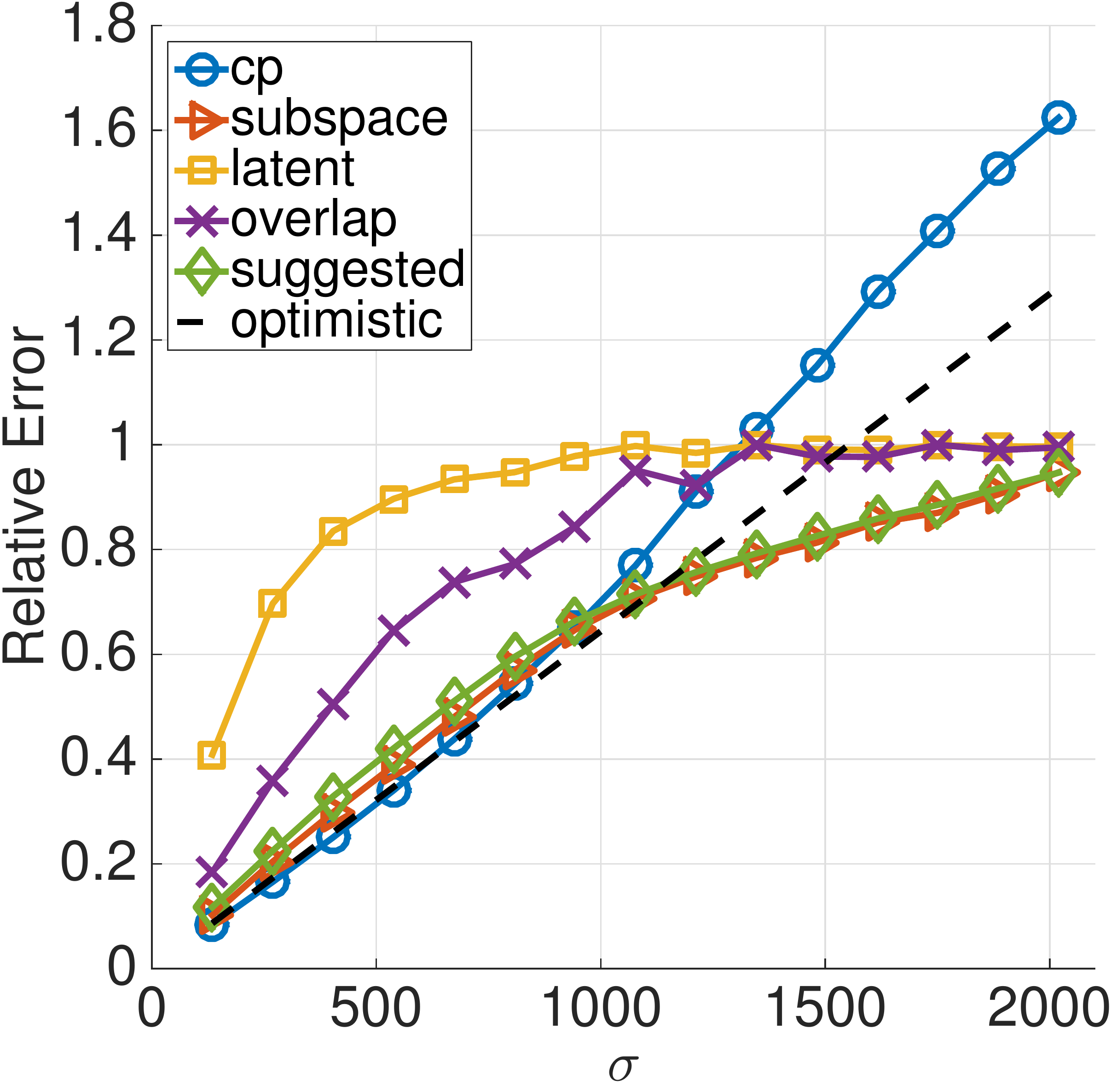}
\label{fig:amino}}
\end{center}
\caption{Tensor denoising. (a) The subspace approach with
    three representative $\lambda$'s on synthetic data. (b)
    Comparison of different methods on synthetic data. (c) Comparison on amino
acids data.}
\label{fig:exp}
\vskip -0.2in
\end{figure} 

\subsection{Amino acids data}
The amino acid dataset \cite{bro1997parafac} is a semi-realistic dataset
commonly used as a benchmark for low rank tensor modeling. It
consists of five laboratory-made samples, each one contains different amounts of tyrosine, tryptophan and phenylalanine. 
The spectrum of their excitation wavelength (250-300 nm) and emission (250-450
nm) are measured by fluorescence, which
gives a $ 5 \times 201 \times 61$ tensor. As the true factors are known to be
these three acids, this data perfectly suits the CP model. 
%
% As for the synthetic dataset, we add random Gaussian noise with standard
%deviation $\sigma$ to the ground truth. The values of $\sigma$ are linearly
%spaced between 130 and 2020.
The true rank is fed into CP and the proposed
approach as $H=3$. We computed the solutions of CP for 20 different random initializations,
and the solutions of other approaches with 20 different values of $\lambda$. For the
subspace and the overlapped approach, $\lambda$'s are logarithmically spaced between
$10^3$ and $10^5$. For the latent approach, $\lambda$'s are logarithmically
spaced between $10^4$ and $10^6$. Again, we include the optimistic scaling \eqref{eq:optimistic} to
put the numbers in context.

Figure \ref{fig:amino} shows the smallest relative error achieved by all methods
we compare. Similar to the synthetic data, both CP and the subspace
norm behaves near ideally, though the relative error of CP can be larger
than 1 due to the lack of regularization. Interestingly the
theoretically suggested scaling of the regularization parameter
$\lambda$ is almost optimal.
% The performance of the overlap
%norm is slightly worse and the latent norm is the worst.

\section{Conclusion}
We have settled a conjecture posed by \cite{RicMon14} and showed that indeed
$O(n^{K/4})$ signal-to-noise ratio is sufficient also for odd order tensors.
Moreover, our analysis shows an interesting two-phase behavior of the error.
This finding lead us to the development of the proposed subspace norm. The
proposed norm is defined with respect to a set of orthonormal matrices
$\wh{\mP}^{(1)},\ldots,\wh{\mP}^{(K)}$, which are estimated by mode-wise singular
value decompositions. We have analyzed the denoising performance of the proposed
norm, and shown that the error can be bounded by the sum of two terms, which can
be interpreted as an approximation error term coming from the first (non-convex)
step, and an estimation error term coming from the second (convex) step.

\bibliography{lucky}

\begin{thebibliography}{10}

\bibitem{AcaDunKolMor11}
E.~Acar, D.~M. Dunlavy, T.~G. Kolda, and M.~M{\o}rup.
\newblock Scalable tensor factorizations for incomplete data.
\newblock {\em Chemometr. Intell. Lab.}, 106(1):41--56, 2011.

\bibitem{AgaNegWai12}
A.~Agarwal, S.~Negahban, and M.~J. Wainwright.
\newblock Noisy matrix decomposition via convex relaxation: Optimal rates in
  high dimensions.
\newblock {\em Ann. Stat.}, 40(2):1171--1197, 2012.

\bibitem{AnaGeHsuKakTel14}
A.~Anandkumar, R.~Ge, D.~Hsu, S.~M. Kakade, and M.~Telgarsky.
\newblock Tensor decompositions for learning latent variable models.
\newblock {\em J. Mach. Learn. Res.}, 15(1):2773--2832, 2014.

\bibitem{BenNad11}
F.~Benaych-Georges and R.~R. Nadakuditi.
\newblock The eigenvalues and eigenvectors of finite, low rank perturbations of
  large random matrices.
\newblock {\em Advances in Mathematics}, 227(1):494--521, 2011.

\bibitem{bro1997parafac}
R.~Bro.
\newblock Parafac. tutorial and applications.
\newblock {\em Chemometr. Intell. Lab.}, 38(2):149--171, 1997.

\bibitem{DeLDeMVan00a}
L.~De~Lathauwer, B.~De~Moor, and J.~Vandewalle.
\newblock A multilinear singular value decomposition.
\newblock {\em SIAM J. Matrix Anal. Appl.}, 21(4):1253--1278, 2000.

\bibitem{DeLDeMVan00b}
L.~De~Lathauwer, B.~De~Moor, and J.~Vandewalle.
\newblock On the best rank-1 and rank-({$R_1, R_2,\ldots, R_N$}) approximation
  of higher-order tensors.
\newblock {\em SIAM J. Matrix Anal. Appl.}, 21(4):1324--1342, 2000.

\bibitem{HilLim13}
C.~J. Hillar and L.-H. Lim.
\newblock Most tensor problems are np-hard.
\newblock {\em J. ACM}, 60(6):45, 2013.

\bibitem{Hit27}
F.~L. Hitchcock.
\newblock The expression of a tensor or a polyadic as a sum of products.
\newblock {\em J. Math. Phys.}, 6(1):164--189, 1927.

\bibitem{HopShiSte15}
S.~B. Hopkins, J.~Shi, and D.~Steurer.
\newblock Tensor principal component analysis via sum-of-squares proofs.
\newblock Technical report, arXiv:1507.03269, 2015.

\bibitem{Has90}
J.~H\r{a}stad.
\newblock Tensor rank is {NP}-complete.
\newblock {\em Journal of Algorithms}, 11(4):644--654, 1990.

\bibitem{HsuKakZha11}
D.~Hsu, S.~M. Kakade, and T.~Zhang.
\newblock Robust matrix decomposition with sparse corruptions.
\newblock {\em Information Theory, IEEE Transactions on}, 57(11):7221--7234,
  2011.

\bibitem{JaiOh14}
P.~Jain and S.~Oh.
\newblock Provable tensor factorization with missing data.
\newblock In {\em Adv. Neural. Inf. Process. Syst. 27}, pages 1431--1439, 2014.

\bibitem{Kol01}
T.~G. Kolda.
\newblock Orthogonal tensor decompositions.
\newblock {\em SIAM Journal on Matrix Analysis and Applications},
  23(1):243--255, 2001.

\bibitem{KolBad09}
T.~G. Kolda and B.~W. Bader.
\newblock Tensor decompositions and applications.
\newblock {\em SIAM Review}, 51(3):455--500, 2009.

\bibitem{KolMay11}
T.~G. Kolda and J.~R. Mayo.
\newblock Shifted power method for computing tensor eigenpairs.
\newblock {\em SIAM J. Matrix Anal. Appl.}, 32(4):1095--1124, 2011.

\bibitem{LauMas00}
B.~Laurent and P.~Massart.
\newblock Adaptive estimation of a quadratic functional by model selection.
\newblock {\em Ann. Stat.}, pages 1302--1338, 2000.

\bibitem{Moe11}
M.~M{\o}rup.
\newblock Applications of tensor (multiway array) factorizations and
  decompositions in data mining.
\newblock {\em Wiley Interdisciplinary Rev.: Data Min. Knowl. Dicov.},
  1(1):24--40, 2011.

\bibitem{MuHuaWriGol14}
C.~Mu, B.~Huang, J.~Wright, and D.~Goldfarb.
\newblock Square deal: Lower bounds and improved relaxations for tensor
  recovery.
\newblock In {\em Proc. ICML '14}. 2014.

\bibitem{RicMon14}
E.~Richard and A.~Montanari.
\newblock A statistical model for tensor pca.
\newblock In {\em Adv. Neural. Inf. Process. Syst. 27}, pages 2897--2905, 2014.

\bibitem{SorVanDeL13}
L.~Sorber, M.~Van~Barel, and L.~De~Lathauwer.
\newblock Optimization-based algorithms for tensor decompositions: Canonical
  polyadic decomposition, decomposition in rank-(l\_r,l\_r,1) terms, and a new
  generalization.
\newblock {\em SIAM Journal on Optimization}, 23(2):695--720, 2013.

\bibitem{tensorlab}
L.~Sorber, M.~Van~Barel, and L.~T. De~Lathauwer.
\newblock {Tensorlab v2.0}.
\newblock {\em \url{http://www. tensorlab. net}}, 2014.

\bibitem{TomSuz13}
R.~Tomioka and T.~Suzuki.
\newblock Convex tensor decomposition via structured {Schatten} norm
  regularization.
\newblock In {\em Adv. Neural. Inf. Process. Syst. 26}, pages 1331--1339. 2013.

\bibitem{TomSuz14}
R.~Tomioka and T.~Suzuki.
\newblock Spectral norm of random tensors.
\newblock Technical report, arXiv:1407.1870, 2014.

\bibitem{TomSuzHayKas11}
R.~Tomioka, T.~Suzuki, K.~Hayashi, and H.~Kashima.
\newblock Statistical performance of convex tensor decomposition.
\newblock In {\em Adv. Neural. Inf. Process. Syst. 24}, pages 972--980. 2011.

\bibitem{Tuc66}
L.~R. Tucker.
\newblock Some mathematical notes on three-mode factor analysis.
\newblock {\em Psychometrika}, 31(3):279--311, 1966.

\bibitem{Ver10}
R.~Vershynin.
\newblock Introduction to the non-asymptotic analysis of random matrices.
\newblock Technical report, arXiv:1011.3027, 2010.

\bibitem{Wed72}
P.-{\AA}. Wedin.
\newblock Perturbation bounds in connection with singular value decomposition.
\newblock {\em BIT Numerical Mathematics}, 12(1):99--111, 1972.

\end{thebibliography}
\bibliographystyle{abbrv}

\appendix
\section{Maximum likelihood estimator}
\label{sec:MLE}
Let $\tY\in\R^{n_1\times\cdots\times n_K}$ be a noisy observed tensor generated as follows:
\begin{align*}
 \tY = \tX^{\ast}+\sigma\tE = \sum_{r=1}^{R}\beta_r \vu_r^{(1)}\circ
 \cdots \circ \vu_r^{(K)} + \sigma\tE,
\end{align*}
where $\tE$ is a noisy tensor whose entries are
i.i.d. normal $\mathcal{N}(0,1)$.

Let $\hat{\tX}_{\text{MLE}}$ be the (intractable) estimator defined as
\begin{align*}
 \hat{\tX}_{\text{MLE}}=\argmin_{\tX}\left(\norm{\tY-\tX}_F^2:\,
 {\rm rank}(\tX)\leq R\right).
\end{align*}

We have the following performance guarantee for $\hat{\tX}_{\text{MLE}}$:
\begin{thm}
\label{thm:mle}
Let $R\leq \min_kn_k/2$. Then there is a constant $c$ such that
\begin{align*}
 \norm{\hat{\tX}_{\text{MLE}}-\tX^{\ast}}_F\leq c\sigma \sqrt{R^K\sum_{k=1}^{K}n_k\log(2K/K_0)+\log(2/\delta)},
\end{align*}
with probability at least $1-\delta$, where $K_0=\log(3/2)$.
\end{thm}

Note that the factor $R^{K}$ in the square root is rather
conservative. In the best case, this factor reduces to linear in $R$ and
this is what we present in Section \ref{sec:exp} as ``optimistic''
ignoring constants and $\delta$; see
Eq. \eqref{eq:optimistic}. 

\begin{proof}[Proof of Theorem \ref{thm:mle}]
Since $\hat{\tX}_{\text{MLE}}$ is a minimizer and $\tX^{\ast}$ is also feasible, we
 have
\begin{align*}
 \norm{\tY-\hat{\tX}_{\text{MLE}}}_F^2\leq\norm{\tY-\tX^{\ast}}_F^2,
\end{align*}
which implies
\begin{align*}
 \norm{\tX^{\ast}-\hat{\tX}_{\text{MLE}}}_F^2&\leq \sigma
 \ip{\tE,\hat{\tX}_{\text{MLE}}-\tX^{\ast}}\\
&\leq \sigma \norm{\tE}_{\text{op}}\norm{\hat{\tX}_{\text{MLE}}-\tX^{\ast}}_{\text{nuc}},
\end{align*}
where
\begin{align*}
 \norm{\tX}_{\text{op}}:=&\sup_{\vu^{(1)},\ldots,\vu^{(K)}
}\bigl\{\!\!\sum_{i_1,i_2,\ldots,i_K}\mathcal{X}_{i_1,i_2,\ldots,i_K}u_{i_1}^{(1)}u_{i_2}^{(2)}\cdots u_{i_K}^{(K)}:\\
&\,\|\vu^{(1)}\|=\|\vu^{(2)}\|=\cdots=\|\vu^{(K)}\|=1\bigr\}
\end{align*}
is the tensor spectral norm and the nuclear norm
\begin{align*}
 \norm{\tX}_{\text{nuc}}:=&\inf_{\vu^{(1)},\ldots,\vu^{(K)}}\bigl\{
\sum_{r}\|\vu_r^{(1)}\|\cdot\|\vu_r^{(2)}\|\cdots\|\vu_r^{(K)}\|:\\
&\,\tX=\sum_{r=1}^{R}\vu_r^{(1)}\circ \cdots \circ \vu_r^{(K)}
\bigr\}
\end{align*}
is the dual of the spectral norm.

Since both $\hat{\tX}_{\text{MLE}}$ and $\tX^{\ast}$ are rank
at most $R$, the difference $\hat{\tX}_{\text{MLE}}-\tX^{\ast}$ is rank
at most $2R$. Moreover, any rank-$R$ CP decomposition with $R\leq \min_kn_k$can be reduced to
an orthogonal CP decomposition with rank at most $R^K$ via the Tucker decomposition
 \cite{Kol01}. Thus, denoting this orthogonal decomposition by
 $\hat{\tX}_{\text{MLE}}-\tX^{\ast}=\sum_{r=1}^{R^K}\tilde{\vu}_r^{(1)}\circ\cdots\circ\tilde{\vu}_r^{(K)}$
 and using
 $\beta_r:=\|\tilde{\vu}_r^{(1)}\|\cdots\|\tilde{\vu}_r^{(K)}\|$, we have
\begin{align*}
 \norm{\hat{\tX}_{\text{MLE}}-\tX^{\ast}}_{\text{nuc}}
&\leq
 \sum_{r=1}^{R^K}\beta_r\leq
 \sqrt{R^K}\sqrt{\sum\nolimits_{r=1}^{R^K}\beta_r^2}\\
&=\sqrt{R^K}\norm{\hat{\tX}_{\text{MLE}}-\tX^{\ast}}_F,
\end{align*}
where the last equality follows because the decomposition is orthogonal.

Finally applying the tail bound for the
spectral norm $\|\tE\|_{\rm op}$ of random Gaussian tensor $\tE$ \cite{TomSuz14}, we obtain what
we wanted.
\end{proof}

\section{Details of optimization}
\label{sec:optimization}
\begin{algorithm}[tb]
   \caption{Tensor denoising via the subspace norm}
   \label{alg:denoise}
\begin{algorithmic}
   \STATE {\bfseries Input:} noisy tensor $\tY$, subspace dimension $H$, regularization constant $\lambda$
   \FOR{$k=1$ {\bfseries to} $K$}
   \STATE $\wh{\mP}^{(k)} \longleftarrow$ top $H$ left singular vectors of
   $\mY_{(k)}$
   \ENDFOR
   \FOR{$k=1$ {\bfseries to} $K$}
   \STATE $\mS^{(k)} \longleftarrow \wh{\mP}^{(1)} \otimes \cdots
 \otimes \wh{\mP}^{(k-1)}\otimes \wh{\mP}^{(k+1)}
 \otimes\cdots  \otimes
   \wh{\mP}^{(K)} $ 
   \ENDFOR
   \STATE {\bfseries Output:} $\wh{\tX} = \argmin_{\tX} \frac{1}{2} \norm{\tY -
   \tX}^2_F + \lambda \norm{\tX}_\lucky$.
\end{algorithmic}
\end{algorithm}

For solving problem \eqref{eq:lucky}, we follow the alternating direction
method of multipliers described in \citet{TomSuzHayKas11}.
We scale the objective function in \eqref{eq:lucky} by $1/\lambda$, and consider the dual problem
\begin{equation}
\begin{aligned}
\label{eq:dual}
\min_{\tD, \{ \mW^{(k)} \}^K_{k=1} } & \;\; \frac{\lambda}{2} \norm{\tD}^2_F - \ip{\tD, \tY}\\
\text{s.t.} & \;\; \max_k \| \mW^{(k)} \| \leq 1, \\
 & \;\; \mW^{(k)} = \mD_{(k)} \mS^{(k)}, \; k = 1, \ldots, K,
\end{aligned}
\end{equation}
where $\tD \in \R^{n_1 \times n_2 \times \cdots \times n_K}$ is the dual tensor that corresponds to the residual in the primal
problem \eqref{eq:lucky}, and $\mW^{(k)}$'s are auxiliary variables introduced
to make the problem equality constrained.

The augmented Lagrangian function of problem \eqref{eq:dual} could be written as
follows:
\begin{align*}
&\hspace{0.25in} L_\eta(\tD, \{ \mW^{(k)} \}_{k=1}^K, \{ \mM^{(k)} \}_{k=1}^K) \\
&= \frac{\lambda}{2} \norm{\tD}^2 - \ip{\tD, \tY} + \sum_{k=1}^K \big( \ip{ \mM^{(k)}, \mD_{(k)}\mS^{(k)} - \mW^{(k)}} \\
&\hspace{0.15in} + \frac{\eta}{2} \| \mD_{(k)}\mS^{(k)} - \mW^{(k)} \|^2_F + \Id_{ \|\cdot\| \leq 1  } ( \mW^{(k)} )\big),
\end{align*}
where $\mM^{(k)}$'s are the multipliers, $\eta$ is the augmenting parameter, and
$\Id_{\|\cdot\| \leq 1}$ is the indicator function of the unit spectral norm
ball. 

We follow the derivation in \cite{TomSuzHayKas11} and conclude that the updates of
$\tD$, $\mM^{(k)}$ and $\mW^{(k)}$ can be computed in closed forms. We further
combine the updates of $\mW^{(k)}$ and other steps so that it needs not to be
explicitly computed. The sum of the products of $\mM^{(k)}$ and
${\mS^{(k)}}^\T$ finally converges
to the solution of the primal problem \eqref{eq:lucky}, see Algorithm \ref{alg:admm}. 

The update for the Lagrangian multipliers $\mM^{(k)}$ ($k=1,\ldots,K$)
is written as  singular value soft-thresholding operator defined as
\[ \text{prox}^{tr}_{\eta} (\mZ ) =  \mP\max(\mSigma-\eta,0)\mQ^\T, \]
where $\mZ=\mP \mSigma\mQ^\T$ is the SVD of $\mZ$.

\begin{algorithm}[tb]
   \caption{ADMM for subspace norm minimization}
   \label{alg:admm}
\begin{algorithmic}
   \STATE {\bfseries Input:} $\tY$,
   $\lambda$, $\mS^{(1)}, \ldots, \mS^{(K)}$, $\eta$, initializations $\tD_0$,
   $\{\mM^{(1)}_0,\ldots,\mM^{(K)}_0\}$
   \STATE $t = 0$
   \REPEAT
   \STATE 
       $ \tD_{t+1} = \frac{1}{\lambda+\eta K}
   \bigg(  \tY + K \eta \tD_t - \sum_k \fold_k \bigl( (2 \mM^{(k)}_t -
   \mM^{(k)}_{t-1} ) {\mS^{(k)}}^\T \bigr) \bigg)
   $
   \FOR{$k=1$ {\bfseries to} $K$}
        \STATE $\mM^{(k)}_{t+1} = \text{prox}^{tr}_{\eta} \left(
        \mM^{(k)}_t + \eta \mD_{(k), t+1} \mS^{(k)}\right)$
   \ENDFOR
   \STATE $ t \leftarrow t + 1 $
   \UNTIL{convergence}
   \STATE {\bfseries Output:} $\wh{\tX} = \sum_{k=1}^K \mM^{(k)}_t
   {\mS^{(k)}}^\T$.
\end{algorithmic}
\end{algorithm}

A notable property of the subspace norm is the computational efficiency.
The update of $\mM^{(k)}$ requires singular value decomposition, which usually
dominates the costs of computation. For problem \eqref{eq:dual}, the size of
$\mM^{(k)}$ is only $n_k \times H^{K-1}$. Comparing with previous approaches,
e.g. the latent approach whose multipliers are $n_k \times \prod_{k'\neq
k}n_{k'}$ matrices, the size of our variables is much
smaller, so the per-iteration cost is reduced. 

\section{Additional experiments}
\label{sec:moreexp}
We report the experimental results when the input rank of CP and the subspace
approach is 
are over-specified, on the same synthetic dataset as Section
\ref{sec:exp}. 
We consider the case where the input rank is 8.

We impose the $\ell_2$ regularizations on the factors of CP. We test 20 values
that are logarithmically spaced between 0.01 and 10 are the regularization
parameter. For each value, we compute 20 solutions with random initializations
and select the one with lowest objective value. 

For the subspace approach, we computed solutions for 20 values of the
regularization parameter that are logarithmically spaced between 1 and 1000.

As before, we report the minimum relative error obtained by the same method.
The results are shown in Figure~\ref{fig:addition}. We include the case the rank
is specified incorrectly for comparison. Clearly, even if the rank is much larger
than the truth, the subspace approach and CP are
robust with proper regularization.

\begin{figure}[tb]
\begin{center}
    \includegraphics[width=\columnwidth]{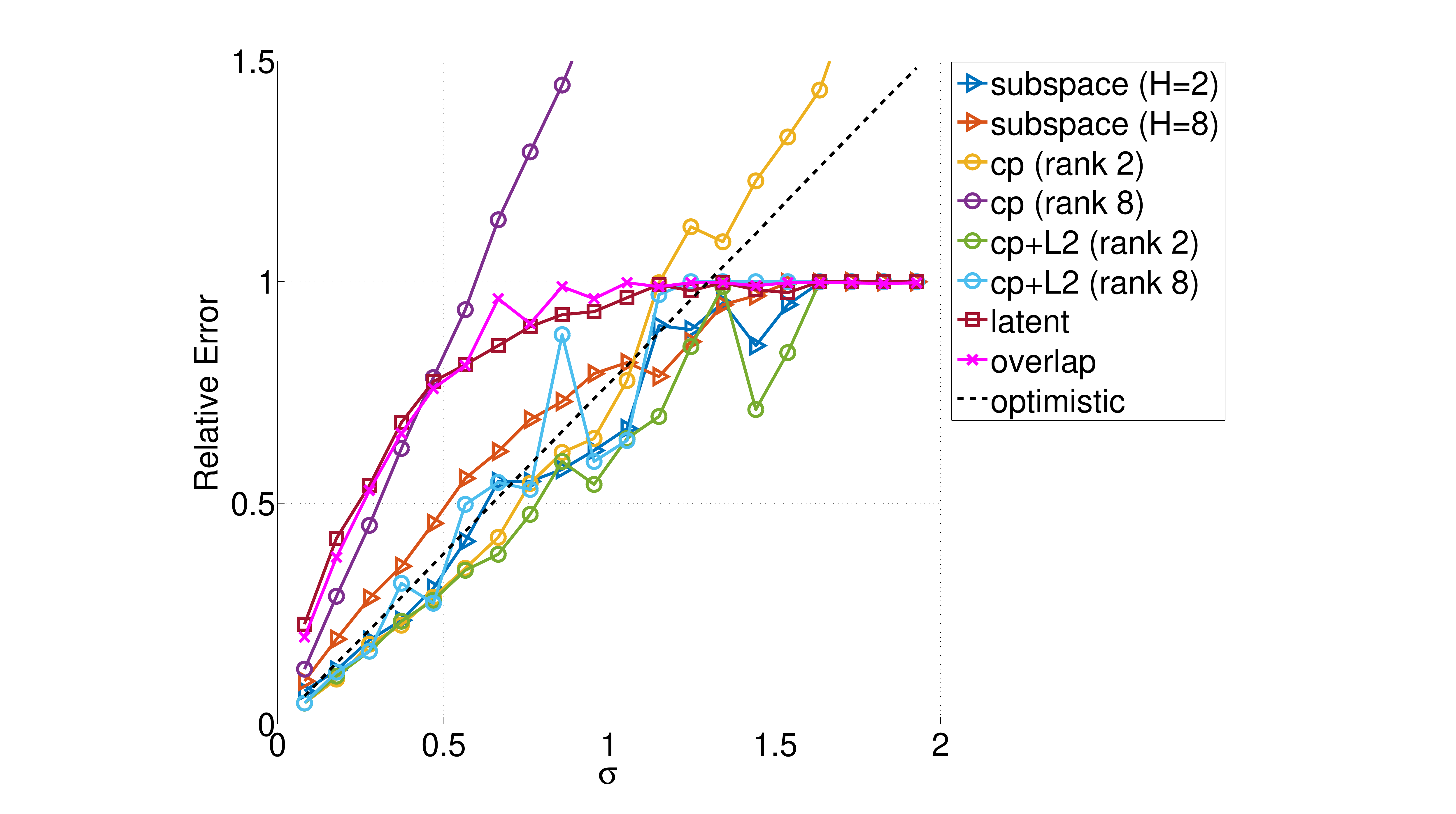}
\end{center}
\caption{Tensor denoising on synthetic dataset when the input rank is larger
than the truth.}
\label{fig:addition}
\vskip -0.2in
\end{figure}

\section{Proofs}
\subsection{Proof of Theorem \ref{thm:leftsv}}
\label{sec:proof-thm-leftsv}
We consider the second moment of $\tilde{\mX}$:
\begin{align*}
 \tilde{\mX}\tilde{\mX}^\T & = \beta^2 \vu \vu^\T + \sigma^2 \mE \mE^\T +
 \beta \sigma (\vu \vv^\T \mE^\T + \mE \vv \vu^\T)\\
 & = \overbrace{\beta^2 \vu \vu^\T + m \sigma^2 \mI}^{\mB} +\\
 & \hspace{1cm}
 \overbrace{\sigma^2 \mE \mE^\T - m \sigma^2
 \mI + \beta\sigma(\vu \vv^\T \mE^\T + \mE \vv \vu^\T)}^{\mG}.
\end{align*}
The eigenvalue decomposition of $\mB$ can be written as
\[ \mB = [ \vu \; \mU_2]
\begin{bmatrix}
\beta^2 + m \sigma^2  & \\ & m\sigma^2 \mI
\end{bmatrix}
\begin{bmatrix}
\vu^\T \\ \mU_2^\T
\end{bmatrix}.
\]

We first show a deterministic lower bound for $|\ip{\hat{\vu},\vu}|$
assuming $\beta^2\geq 2\|\mG\|$, where $\hat{\vu}$ is the leading
eigenvector  of $\tilde{\mX}\tilde{\mX}^\T$. Then we bound the spectral
norm $\|\mG\|$ of the noise term (Lemma \ref{lemma:operator-norm}) and derive the sufficient condition for $\beta$.

Let $\hat{\vu}$ be the leading eigenvector of
 $\tilde{\mX}\tilde{\mX}^\T$ with eigenvalue $\hat{\lambda}$, $\vr = \mB
 \hat{\vu} - \hat{\lambda} \hat{\vu} = - \mG\hat{\vu}$.
We have $\mU^\T_2 \vr = (m \sigma^2 - \hat{\lambda}) \mU^\T_2
\hat{\vu}$. Hence, for all $\beta^2 >  2 \| \mG \|$, it holds that
\begin{align*}
 \abs{\sin(\hat{\vu}, \vu)}& = \| \mU^\T_2 \hat{\vu}\|_2 = \frac{ \|
 \mU^\T_2 \vr\|_2}{ \hat{\lambda} - m\sigma^2 } 
 \leq \frac{\| \mG \|}{ \beta^2 - \| \mG \| } \leq \frac{2\|\mG\|}{\beta^2},
\end{align*}
where we used $\|\mU^\T_2\vr\|_2 = \| \mU^\T_2 \mG \hat{\vu}\|_2 \leq \|
 \mG \|$, and $\hat{\lambda} \geq \vu^\T \tilde{\mX} \tilde{\mX}^\T
 \vu^\T \geq \beta^2 + m\sigma^2 -\|\mG\|$.
Therefore,
 \[ \abs{\ip{\hat{\vu}, \vu }} = \abs{ \cos(\hat{\vu}, \vu) } \geq \sqrt{1 -
 \frac{4\|\mG\|^2}{\beta^4}} \geq 1 - \frac{4\|\mG\|^2}{\beta^4},   \]
if $\beta^2\geq 2\|\mG\|$.

 It follows from Lemma \ref{lemma:operator-norm} (shown below) that
 \[
     \| \mG \| \leq
 \begin{cases}
     2\bar{C} \sigma^2 \sqrt{mn}, & \text{if} \quad \beta/\sigma <\sqrt{m},\\
     2\bar{C} \beta \sigma \sqrt{n}, & \text{otherwise},
 \end{cases}
 \]
 where $\bar{C}$ is a universal constant with probability at least $1-4e^{-n}$.

Now consider the first case ($\beta/\sigma<\sqrt{m}$) and assume
$\beta^2\geq 4\bar{C}\sigma^2\sqrt{mn}\geq 2\|\mG\|$. Note that this
case only arises when $\sqrt{m}\geq 4\bar{C}\sqrt{n}$. Denoting
$C=16\bar{C}^2$, we obtain the first case in the theorem. Next, consider
the second case ($\beta/\sigma\geq\sqrt{m}$). If $\sqrt{m}\geq
4\bar{C}\sqrt{n}$ as above, we have $\beta/\sigma\geq 4\bar{C}\sqrt{n}$, which
implies $\beta^2\geq 2\|\mG\|$ and we obtain the second case in the
theorem. On the other hand, if $\sqrt{m}<4\bar{C}\sqrt{n}$, we require
$\beta/\sigma\geq4\bar{C}\sqrt{n}$ to obtain the last case in the
theorem.

\begin{lemma}
\label{lemma:operator-norm}
Let $\mG$ be constructed as in Theorem \ref{thm:leftsv}. If $m\geq n$, 
there exists an universal constant $\bar{C}$ such that 
%it holds that
\[ \| \mG \| \leq \bar{C} \sigma^2 \left(\sqrt{mn} +\sqrt{n(\beta/\sigma)^2}\right), \]
with probability at least $1-4e^{-n}$.
\end{lemma}
\begin{proof}
The proof is an $\epsilon$-net argument.
Let
\[ 
    \lambda = 
%2 \sigma^2 \left( \sqrt{\log(4/\theta)}\left(\sqrt{m} + \sqrt{8 (\beta/\sigma)^2}
%    \right) + \log(4/\theta) \right ) ,
2\sigma^2 \left(\sqrt{4mn} + 4n + \sqrt{8n(\beta/\sigma)^2} \right).
\]
The goal is to control $ | \vx^\T \mG \vx |$ for all the vectors $\vx$ on the unit Euclidean sphere
$\tS^{n-1}$. In order to do this, we first bound the probability of the tail event $|\vx^\T \mG \vx| > \lambda$,
for any fixed $\vx \in \tS^{n-1}$.
Then we bound the probability that $|\vx^\T \mG \vx| > \lambda $ for all the vectors in a 
$\epsilon$-net $\tN_{\epsilon}$. Finally, we establish the connection
between $\sup_{\vx \in \tN_{\epsilon} } | \vx^\T \mG \vx |$ and $\| \mG \|$.

To bound $\P( |\vx^\T \mG \vx| > \lambda)$ for a fix $\vx \in
 \tS^{n-1}$, we expand $\vx^\T \mG \vx$ as
$$\vx^\T\mG \vx=\sigma^2 (\| \vz \|^2 - m) + 2\beta\sigma(\vu^\T\vx)
 \gamma, $$
where $\vz = \mE^\T\vx$ and $\gamma=\vv^\T\vz$. Since  $\vz \sim \tN(0,
 \mI)$, we can see that $\|\vz\|^2$ is $\chi^2$ distributed with $m$
degrees of freedom and $\gamma \sim\tN(0, 1)$.
%Let $\lambda_1 =2\sigma^2 (\sqrt{m \log(4/\theta)} + \log(4/\theta) ) $, $\lambda_2 = 2\beta \sigma
%\sqrt{2\log(4/\theta)}$. 

First we bound the deviation of the $\chi^2$ term. By the corollary of
 Lemma 1 in \citet{LauMas00}, we have
\begin{align}
\label{ineq:chi2}
    \P(\big| \| \vz \|^2 - m \big| > \lambda_1) \leq  2 e^{-4n},% \frac{\theta}{2}.
\end{align}
where $\lambda_1=2(\sqrt{4mn}+4n)$.

Next we bound the deviation of the Gaussian term. 
Using the Gaussian tail inequality, we have
\begin{align}
\label{ineq:gaussian}
\P\left( | \gamma | >\lambda_2 \right) \leq 2e^{-4n}, % \frac{\theta}{2}.\]
\end{align}
where $\lambda_2=\sqrt{8n}$.

Combining inequalities \eqref{ineq:chi2} and \eqref{ineq:gaussian}, we have
\begin{align*}
&\P(|\vx^\T \mG \vx| > \lambda)\\
&\leq\; \P\left( \sigma^2 \big|  \| \vz \|^2 - m \big|  +
| 2\beta\sigma(\vu^\T\vx)  \gamma | > \sigma^2\lambda_1 + 2\beta\sigma\lambda_2 \right)\\
&\leq \; \P\left( \big| \| \vz \|^2 - m \big| > \lambda_1 \lor
    |\gamma| >  \lambda_2\right)\\
&\leq \; \P\left( \big| \| \vz \|^2 - m \big| > \lambda_1 \right) +  \P \left(
    |\gamma| >  \lambda_2\right)\\
&\leq \; 4e^{-4n},
\end{align*}
where the second to last line follows from the union bound.

Furthermore, using Lemma 5.2 and 5.4 of
\citet{Ver10}, for any $\epsilon \in [0, 1)$, it holds that 
\[
    |\tN_\epsilon| \leq (1+ 2/\epsilon)^n, 
\]
and
\begin{equation*}
    \| \mG \| \leq (1-2\epsilon)^{-1} \sup_{\vx \in \tN_\epsilon} | \vx^\T \mG
    \vx |.
\end{equation*}
Taking the union bound over all the vectors in $\tN_{1/4}$, we obtain
\[
\P\left(\sup_{\vx \in \tN_{1/4} } |\vx^\T \mG \vx| > \lambda\right) \leq
| \tN_{1/4} | 4e^{-4n} < 4e^{-n}.
\]

Finally, the statement is obtained by noticing that 
$n \leq m$.\hfill$\Box$
%$ \sqrt{4mn}+4n\leq 6\sqrt{mn}$ and $\sqrt{8}\leq 6$.

We prove a more general version
of the theorem that allows the signal part to be rank $R$ in Appendix \ref{sec:leftsv_rankR}.

\subsection{Proof of Lemma \ref{lem:dualnorm}}
\label{sec:proof-dualnorm}
\begin{proof} 
By definition,
\[
\begin{aligned}
    \norm{\tY}_{\dual}
   % & =&&\sup_{ \tX } \ip{\tX, \tY} \\
   % & &&\text{s.t.} \sum_k \|\mM^{(k)}\|_* \leq 1, \tX = \sum_k
   % \fold_k(\mM^{(k)}\mS^{(k)})^\T \\
    & =&&\sup_{ \{ \mM^{(k)} \}^K_{k=1} } \ip{\tY, \sum_{k=1}^K \fold_k(\mM^{(k)}{\mS^{(k)}}^\T)}\\
    & &&\text{s.t.} \sum_{k=1}^K \|\mM^{(k)}\|_* \leq 1\\
    & =&&\sup_{\{ \mM^{(k)} \}^K_{k=1} } \sum_{k=1}^K \ip{\mY_{(k)}\mS^{(k)}, \mM^{(k)}}\\
    & &&\text{s.t.} \sum_{k=1}^K \|\mM^{(k)}\|_* \leq 1\\
    & =&&\max_k \|\mY_{(k)}\mS^{(k)}\|,
\end{aligned}
\]
where we used the H\"{o}lder inequality in the last line.
\end{proof}

\subsection{Proof of Theorem \ref{thm:lucky}}
\label{sec:proof-thm-lucky}
First we decompose the error as
\begin{align*}
 \norm{\tX^{\ast}-\hat{\tX}}_F\leq \norm{\tX^{\ast}-\tX_p}_F + \norm{\tX_p-\hat{\tX}}_F.
\end{align*}
The first term is an approximation error that depends on the choice of
the subspace $\mS^{(k)}$. The second term corresponds to an estimation
error and we analyze the second term below.

Since $\hat{\tX}$ is the minimizer of \eqref{eq:estimator} and $\tX_p$ is
feasible,
\begin{align*}
\hspace{-1cm}
\frac{1}{2}\norm{\tY-\hat{\tX}}_F^2+\lambda\sum_{k=1}^{K}\|\hat{\mM}^{(k)}\|_{\ast}
\leq \frac{1}{2}\norm{\tY-\tX_p}_F^2+\lambda\sum_{k=1}^{K}\|\mM_p^{(k)}\|_{\ast},
\end{align*}
from which we have
\begin{align}
\label{ineq:step1}
 \frac{1}{2}\norm{\tX_p-\hat{\tX}}_F^2\leq
\norm{\tY-\tX_p}_{\lucky^\ast}\norm{\tX_p-\hat{\tX}}_{\lucky}
+\lambda\sum_{k=1}^{K}\left(\|\mM_p^{(k)}\|_{\ast}-\|\hat{\mM}^{(k)}\|_{\ast}\right).
\end{align}

Next we define $\mDelta_k:=\hat{\mM}^{(k)}-\mM_p^{(k)}\in\R^{n_k\times
H^{K-1}}$ and define its orthogonal decomposition
$\mDelta_k=\mDelta_k'+\mDelta_k''$ as
\begin{align*}
 \mDelta_k'' := (\mI_{n_k}-\mP_{U_p})\mDelta_k(\mI_{H^{K-1}}-\mP_{V_p}),
\end{align*}
where $\mP_{U_p}$ and $\mP_{V_p}$ are
projection matrices to the column and row spaces of $\mM_p^{(k)}$,
respectively, and $\mDelta_k':=\mDelta_k-\mDelta_k''$.

The above definition allows us to decompose $\|\hat{\mM}^{(k)}\|_{\ast}$
as follows:
\begin{align}
 \|\hat{\mM}^{(k)}\|_{\ast}&=\|\mM_p^{(k)}+\mDelta_k''+\mDelta_k'\|_{\ast}\notag\\
 \label{ineq:step2}
&\geq \|\mM_p^{(k)}\|_{\ast}+\|\mDelta_k''\|_{\ast}-\|\mDelta_k'\|_{\ast}.
\end{align}

Moreover,
\begin{align}
\label{ineq:step3}
 \norm{\tX_p-\hat{\tX}}_{\lucky}\leq
 \sum_{k=1}^{K}\|\mDelta_k\|_{\ast}\leq \sum_{k=1}^{K}\left(\|\mDelta_k'\|_{\ast}+\|\mDelta_k''\|_{\ast}\right)
\end{align}

Combining inequalities \eqref{ineq:step1}--\eqref{ineq:step3}, we have
\begin{align}
\label{ineq:step4}
 \frac{1}{2}\norm{\tX_p-\hat{\tX}}_F^2\leq
(\norm{\tY-\tX_p}_{\lucky^{\ast}}+\lambda)\sum_{k=1}^{K}\|\mDelta_k'\|_{\ast}+(\norm{\tY-\tX_p}_{\lucky^{\ast}}-\lambda)\sum_{k=1}^{K}\|\mDelta_k''\|_{\ast}.
\end{align}

Since
\begin{align*}
 \norm{\tY-\tX_p}_{\lucky^\ast}\leq \sigma\norm{\tE}_{\lucky^{\ast}}+\norm{\tX^{\ast}-\tX_p}_{\lucky^{\ast}},
\end{align*}
if $\lambda\geq
\sigma\norm{\tE}_{\lucky^{\ast}}+\norm{\tX^{\ast}-\tX_p}_{\lucky^{\ast}}$, the
second term in the right-hand side of inequality \eqref{ineq:step4} can
be ignored and we have
\begin{align}
 \frac{1}{2}\norm{\tX_p-\hat{\tX}}_F^2&\leq
 2\lambda\sum_{k=1}^{K}\|\mDelta_k'\|_{\ast}\notag\\
&\leq 2\lambda \sum_{k=1}^{K}\sqrt{2r_k}\|\mDelta_k'\|_F\notag\\
&\leq 2\lambda \sum_{k=1}^{K}\sqrt{2r_k}\|\mDelta_k\|_F\notag\\
\label{ineq:step5}
&\leq 2\sqrt{2}\lambda\sqrt{\sum_{k=1}^{K}r_k}\sqrt{\sum_{k=1}^{K}\|\mDelta_k\|_F^2},
\end{align}
where in the second line we used a simple observation that 
 ${\rm rank}(\mDelta_k')\leq 2r_k$.

Next, we relate the norm $\norm{\tX_p-\hat{\tX}}_F$ to the sum
$\sum_{k=1}^{K}\|\mDelta_k\|_F^2$ in the right-hand side of inequality
\eqref{ineq:step5}.

First suppose that
$\sum_{k=1}^{K}\|\mDelta_k\|_F^2\leq\norm{\tX_p-\hat{\tX}}_F^2$. Then
 from inequality \eqref{ineq:step5}, we have
\begin{align*}
\norm{\tX_p-\hat{\tX}}_F\leq 4\sqrt{2}\lambda\sqrt{\sum_{k=1}^{K}r_k}
\end{align*}
by dividing both sides by $\norm{\tX_p-\hat{\tX}}_F$.

On the other hand, if $\norm{\tX_p-\hat{\tX}}_F^2\leq
\sum_{k=1}^{K}\|\mDelta_k\|_F^2$, we use the following lemma
\begin{lemma}
\label{lem:coherence}
Suppose $\{\mM_p^{(k)}\}_{k=1}^{K},\{\hat{\mM}^{(k)}\}_{k=1}^{K}\in
 \mathcal{M}(\rho)$, and $\mS^{(k)}$ is constructed as a Kronecker product of $K-1$
 ortho-normal matrices $\wh{\mP}^{(\ell)}$ as
 $\mS^{(k)}=\wh{\mP}^{(k-1)}\otimes\cdots\otimes \wh{\mP}^{(k+1)}$,
 where $(\wh{\mP}^{(\ell)})^\T\wh{\mP}^{(\ell)}=\mI_{H}$ for
 $\ell=1,\ldots,K$. Then for
$\tX_p=\sum_{k=1}^{K}\fold_k\left(\mM_p^{(k)}{\mS^{(k)}}^\T\right)$
 and
 $\hat{\tX}=\sum_{k=1}^{K}\fold_k\left(\hat{\mM}^{(k)}{\mS^{(k)}}^\T\right)$,
the following inequality holds:
\begin{align}
\label{ineq:step6}
 \frac{1}{2}\sum_{k=1}^{K}\|\mDelta_k\|_F^2 \leq \frac{1}{2}\norm{\tX_p-\hat{\tX}}_F^2+\rho\max_k(\sqrt{n_k}+\sqrt{H^{K-1}})\sum_{k=1}^{K}\|\mDelta_k\|_{\ast}.
\end{align}
\end{lemma}
\begin{proof}
The proof is presented in Section \ref{sec:proof-lemma-coherence}.
\end{proof}

Combining inequalities \eqref{ineq:step4} and \eqref{ineq:step6}, we
 have
\begin{align*}
\frac{1}{2}\sum_{k=1}^{K}\|\mDelta_k\|_F^2\leq  &
\left(\norm{\tY-\tX_p}_{\lucky^{\ast}}+\rho\max_k(\sqrt{n_k}+\sqrt{H^{K-1}})+\lambda\right)\sum_{k=1}^{K}\|\mDelta_k'\|_{\ast}\\
&\quad+\left(\norm{\tY-\tX_p}_{\lucky^{\ast}}+\rho\max_k(\sqrt{n_k}+\sqrt{H^{K-1}})-\lambda\right)\sum_{k=1}^{K}\|\mDelta_k''\|_{\ast}.
\end{align*}
Thus if we take $\lambda\geq
\sigma\norm{\tE}_{\lucky^\ast}+\norm{\tX^\ast-\tX_p}_{\lucky^\ast}+\rho\max_k(\sqrt{n_k}+\sqrt{H^{K-1}})$,
the second term in the right-hand side can be ignored and following the
derivation leading to inequality \eqref{ineq:step5} and dividing both
sides by $\sqrt{\sum_{k=1}^{K}\|\mDelta_k\|_F^2}$, we have
\begin{align*}
 \norm{\tX_p-\hat{\tX}}_F \leq \sqrt{\sum_{k=1}^{K}\|\mDelta_k\|_F^2}
\leq 4\sqrt{2}\lambda\sqrt{\sum_{k=1}^{K}r_k},
\end{align*}
where the first inequality follows from the assumption.

The final step of the proof is to bound the norm
 $\norm{\tE}_{\lucky^\ast}$ with sufficiently high probability. By
Lemma \ref{lem:dualnorm},
\begin{align*}
 \norm{\tE}_{\lucky^\ast}=\max_k\|\mE_{(k)}\mS^{(k)}\|.
\end{align*}
Therefore, taking the union bound, we have
\begin{align}
\label{ineq:step7}
\P\left(\max_k\|\mE_{(k)}\mS^{(k)}\|\geq t\right) \leq
 \sum_{k=1}^{K}\P\left(\|\mE_{(k)}\mS^{(k)}\|\geq t\right).
\end{align}
Now since each $\mE_{(k)}\mS^{(k)}\in\R^{n_k\times H^{K-1}}$ is a random matrix with
 i.i.d. standard Gaussian entries,
\begin{align*}
 \P\left(\|\mE_{(k)}\mS^{(k)}\|\geq \sqrt{n_k}+\sqrt{H^{K-1}}+t\right)\leq \exp(-t^2/(2\sigma^2)).
\end{align*}
Therefore, choosing
$t=\max_k(\sqrt{n_k}+\sqrt{H^{K-1}})+\sqrt{2\log(K/\delta)}$ in inequality \eqref{ineq:step7}, we have
\begin{align*}
\max_k\|\mE_{(k)}\mS^{(k)}\|\leq \max_k(\sqrt{n_k}+\sqrt{H^{K-1}})+\sqrt{2\log(K/\delta)},
\end{align*}
with probability at least $1-\delta$. Plugging this into the condition
for the regularization parameter $\lambda$, we obtain what we wanted.

\subsection{Proof of Lemma \ref{lemma:span}}
\label{sec:proof-lemma-span}
\begin{proof}

\begin{enumerate}[i)]
    \item Let $\otimes_{k'\in [K] \backslash k} \mU^{(k')}$ denote
              $\mU^{(1)}  \otimes \cdots \otimes  \mU^{(k-1)} \otimes
              \mU^{(k+1)} \otimes \cdots \otimes \mU^{(K)}$.
         We have
         \[
             \begin{aligned}
                 \mX^*_{(k)} & = \mU^{(k)} \mC_{(k)} \left( \otimes_{k'\in [K]
         \backslash k} \mU^{(k')} \right)^\T \\
                  & = \mU^{(k)} \mC_{(k)} \left( \otimes_{k'\in [K]
 \backslash k} ( \mU^{(k')} )^\T \right) \\
             & = \mP^{(k)} \mLambda^{(k)} (  \mQ^{(k)} )^\T.
             \end{aligned}
         \]
         Because of the minimality of the Tucker decomposition
         \eqref{eq:rank-r-tucker},
         $\mX^*_{(k)}$, $\mC_{(k)}$ and $\mU^{(k)}$ are all of rank $R$, for all
         $k \in [K]$.
         Therefore, both $\mC_{(k)}$ and $\otimes_{k'\in [K] \backslash k}
         (\mU^{(k')} )^\T$ have full row rank.
         
         Hence, $\mC_{(k)}$ has a Moore-Penrose pseudo inverse
         $\mC_{(k)}^\dagger$ such that $\mC_{(k)} \mC_{(k)}^\dagger = \mI$, and
         so does $\otimes_{k'\in [K] \backslash k}
         (\mU^{(k')} )^\T$.
         As a result, we have
         \[
              \mU^{(k)}  = \mP^{(k)} \mLambda^{(k)} ( \mQ^{(k)} )^\T
              \left( \otimes_{k'\in [K] \backslash k} ( \mU^{(k')} )^\T \right)^\dagger \mC_{(k)}^\dagger.
         \]

 \item Similarly, we have
     \[ \mQ^{(k)} \mLambda^{(k)} (\mP^{(k)})^\T = \left( \otimes_{k' \in [K] \backslash k} \mU^{(k')}  \right) \mC^\T_{(k)} ( \mU^{(k)} )^\T. \]
     By the definition of SVD, $\mLambda$ is invertible and $(\mP^{(k)})^\T \mP^{(k)} = \mI$.
     Hence, 
     \[ \mQ^{(k)}= \left( \otimes_{k' \in [K] \backslash k} \mU^{(k')}  \right)
     \mC^\T_{(k)} ( \mU^{(k)} )^\T \mP^{(k)} ( \mLambda^{(k)})^{-1}.\]
     This means
     $ \mQ^{(k)} \in \text{span}
      \left( \otimes_{k' \in [K] \backslash k} \mU^{(k')}  \right) $ and we then
      conclude
     $ \mQ^{(k)} \in \text{span}
     \left( \otimes_{k' \in [K] \backslash k} \mP^{(k')}  \right) $ using (i).
      
\end{enumerate}
\end{proof}

\subsection{Proof of Lemma \ref{lem:coherence}}
\label{sec:proof-lemma-coherence}
Expanding $\tX_p$ and $\hat{\tX}$, we have
\begin{align*}
 \norm{\tX_p-\hat{\tX}}_F^2&=\norm{\textstyle\sum_{k=1}^{K}\fold_k\bigl(\mDelta_k{\mS^{(k)}}^\T\bigr)}_F^2\\
&=\sum_{k=1}^{K}\|\mDelta_k\|_F^2+\sum_{k\neq \ell}\ip{\fold_k(\mDelta_k{\mS^{(k)}}^\T),\fold_{\ell}(\mDelta_{\ell}{\mS^{(\ell)}}^\T)}\\
&=\sum_{k=1}^{K}\|\mDelta_k\|_F^2+\sum_{k\neq
 \ell}\ip{\fold_k(\mDelta_k)\times_{k'\neq k}\wh{\mP}^{(k')},\fold_{\ell}(\mDelta_{\ell})\times_{\ell'\neq\ell}\wh{\mP}^{(\ell')}}\\
&=\sum_{k=1}^{K}\|\mDelta_k\|_F^2+\sum_{k\neq \ell}\ip{\fold_k(\mDelta_k)\times_{\ell}\wh{\mP}^{(\ell)},\fold_{\ell}(\mDelta_{\ell})\times_{k}\wh{\mP}^{(k)}}\\
&=\sum_{k=1}^{K}\|\mDelta_k\|_F^2-\sum_{k\neq
 \ell}\ip{\mDelta_k(\mI_{H}\otimes \cdots \otimes \wh{\mP}^{(\ell)}\otimes \cdots\otimes\mI_{H})^\T,\wh{\mP}^{(k)}(\fold_{\ell}(\mDelta_{\ell}))_{(k)}}\\
&\geq\sum_{k=1}^{K}\|\mDelta_k\|_F^2-\sum_{k\neq
 \ell}\|\mDelta_k\|_{\ast}\cdot\|(\fold_{\ell}(\mDelta_{\ell}))_{(k)}\|\\
&\geq\sum_{k=1}^{K}\|\mDelta_k\|_F^2-2{\rho}\max_k(\sqrt{n_k}+\sqrt{H^{K-1}})\sum_{k=1}^{K}\|\mDelta_k\|_{\ast},
\end{align*}
from which the lemma holds. Here we regarded
$\fold_k(\mDelta_k\mS^{(k)})$ as a Tucker decomposition with the core
tensor $\fold_k(\mDelta_k)$ and factor matrices $\wh{\mP}^{(k')}$ for
$k'\neq k$. Most of the factors except for $k$ and $\ell$ cancel out when
calculating the inner product between two such tensors in the third
line, because
$(\wh{\mP}^{(k')})^\T\wh{\mP}^{(k')}=\mI_{H}$. After unfolding
the inner product at the $k$th mode in the fifth
line, we notice that
a multiplication by an ortho-normal matrix does not affect the nuclear norm
or the spectral norm. In the last line we used $\{\mDelta_k\}_{k=1}^{K}\in\mathcal{M}(2\rho)$, which follows
from the assumption that both $\{\mM_p^{(k)}\}_{k=1}^{K},\{\hat{\mM}^{(k)}\}_{k=1}^{K}\in\mathcal{M}(\rho)$.
\end{proof}

\clearpage
\section{Generalization of Theorem \ref{thm:leftsv} to the higher rank case}
\label{sec:leftsv_rankR}
\begin{thm}
    \label{thm:leftsv_rankR}
    Suppose that $\mX = \sum_{r=1}^R \beta_r \vu_r \vv_r^\T$, where $\vu_1,
    \ldots, \vu_R \in \R^n$ and $\vv_1, \ldots, \vv_R \in \R^m$ are unit
    orthogonal vectors respectively. Let $\tilde{\mX} = \mX + \sigma
    \mE$ be the noisy observation of $\mX$.
    There exists an universal constant $C$ such that
    with probability at least $1 - 3e^{-n}$, 
    if $m /n \geq C (\beta_1 / \beta_R)^4$, then 
    \begin{align*}
        \abs{ \cos(\hat{\mU}, \mU) } \geq 
        \begin{cases}
            1 - \dfrac{Cmn}{(\beta_R / \sigma)^4},  & \text{if} \quad
            \dfrac{\beta_1}{ \sqrt{m} }<  \sigma \leq \dfrac{\beta_R}{ (Cmn)^\frac{1}{4} }, \\
            1 - \dfrac{Cn (\beta_1 / \beta_R)^2 }{ (\beta_R / \sigma)^2},
            & \text{if} \quad \sigma \leq \dfrac{\beta_1}{\sqrt{m}},
        \end{cases}
    \end{align*}
    otherwise,
        $\abs{ \cos(\hat{\mU}, \mU) } \geq 1 - \dfrac{Cn (\beta_1 /
            \beta_R)^2 }{ (\beta_R / \sigma)^2}$
    if $ \sigma \leq \beta^2_R / (Cn)^\frac{1}{2} \beta_1$.
\end{thm}

Suppose that $\mX = \sum_{r=1}^R \beta^2_r \vu_r \vv_r^\T$ and $\tilde{\mX} = \mX + \sigma
\mE$. We consider the second moment of $\tilde{\mX}$:
\begin{align*}
    \tilde{\mX}\tilde{\mX}^\T & = \sum_{r=1}^R \beta^2_r \vu_r \vu_r^\T +
    \sigma \left( \sum_{r=1}^R \beta_r \big(\vu_r \vv^\T_r \mE^\T + \mE \vv_r \vu_r^\T \big) \right)
    + \sigma^2 \mE \mE^\T\\
    & = \overbrace{ \sum_{r=1}^R \beta^2_r \vu_r \vu_r^\T + m \sigma^2 \mI}^{\mB} +\\
 & \hspace{1cm}
 \overbrace{\sigma^2 \mE \mE^\T - m \sigma^2
     \mI + \sigma\left( \sum_{i=1}^R \beta_r \big( \vu_r \vv_r^\T \mE^\T + \mE
         \vv_r \vu^\T_r \big) \right) }^{\mG}.
\end{align*}
The eigenvalue decomposition of $\mB$ can be written as
\[ \mB = [ \mU \; \mU_2]
\begin{bmatrix}
\Sigma + m \sigma^2 \mI  & \\ & m\sigma^2 \mI
\end{bmatrix}
\begin{bmatrix}
\mU^\T \\ \mU_2^\T
\end{bmatrix},
\]
where $\mU \in \R^{n \times R}$ and $\Sigma = \text{diag}(\beta^2_1, \ldots,
\beta^2_R)$. Similarly, the eigenvalue decomposition of $\tilde{\mX}\tilde{\mX}^\T$ can be written as
\[ \tilde{\mX}\tilde{\mX}^\T = [ \hat{\mU} \; \hat{\mU}_2]
\begin{bmatrix}
    \hat{\Sigma} & \\ & \hat{\Sigma}'
\end{bmatrix}
\begin{bmatrix}
    \hat{\mU}^\T \\ \hat{\mU}_2^\T
\end{bmatrix},
\]
where $\hat{\Sigma} = \text{diag}(\hat{\lambda}_1, \ldots, \hat{\lambda}_R)$ and
$\hat{\Sigma}' = \text{diag}(\hat{\lambda}_{R+1}, \ldots, \hat{\lambda}_n)$ 
s.t. $\hat{\lambda}_1 \geq \cdots \geq \hat{\lambda}_n$
are the eigenvalues of $\tilde{\mX}\tilde{\mX}^T$.

We first show a deterministic lower bound for $\abs{\sin(\hat{\mU},\mU)}$ assuming
$\beta^2_R \geq 2\|\mG\|$. Then we bound the spectral norm $\|\mG\|$ of the noise
term (Lemma \ref{lemma:operator-norm}) and derive the sufficient condition for
$\beta^2_R$.

The maximum singular value of $m\sigma^2 \mI$ is $m \sigma^2$.
The minimum singular value of $\hat{\Sigma}$ is $|\hat{\lambda}_R|$.
By Wely's theorem,  
$ \| \mG \| \geq \abs{ \hat{\lambda}_R - \beta^2_R - m\sigma^2 }$, which means
\[ \hat{\lambda}_R \geq  m\sigma^2 + \beta^2_R - \| \mG \| . \]

Let $\mR = \mG \hat{\mU}$. Since $\beta^2_R \geq 2\| \mG \|$, 
we can apply the Wedin theorem and obtain
\[
    \abs{\sin(\hat{\mU}, \mU)}  = \| \mU^\T_2 \hat{\mU} \|
    \leq \frac{ \| \mR \|} { \beta^2_R - \| \mG  \|}
    = \frac{ \| \mG \hat{\mU}\|} { \beta^2_R - \| \mG  \|}
    \leq \frac{ \| \mG \| }{ \beta^2_R  - \| \mG \|}
    \leq \frac{ 2 \| \mG \| }{ \beta^2_R },
\]
where we used the property that the spectral norm is sub-multiplicative and
$\| \hat{\mU} \| = 1$ in the second to last step.

Therefore,
 \[ \abs{ \cos(\hat{\mU}, \mU) } \geq \sqrt{1 - \frac{4\|\mG\|^2}{\beta^4_R}}
 \geq 1 - \frac{4\|\mG\|^2}{\beta^4_R},   \]
if $\beta^2_R \geq 2\|\mG\|$. It follows from Lemma
\ref{lemma:operator-norm-rankR} (shown below) that with probability at least
$1 - 3e^{-n}$
 \[
     \| \mG \| \leq
 \begin{cases}
     2\bar{C} \sigma^2 \sqrt{mn}, & \text{if} \quad \beta_1/\sigma <\sqrt{m},\\
     2\bar{C} \sigma \sqrt{n} \beta_1, & \text{otherwise},
 \end{cases}
 \]
where $\bar{C}$ is an universal constant.

Let $C = 16\bar{C}^2$.  
Now consider the first situation where $ m / n > C ( \beta_1 / \beta_R)^4 $. 
If $\sigma > \frac{ \beta_1}{ \sqrt{m}}$, we have $\| \mG
\| \leq \frac{\sigma^2}{2} (Cmn)^{\frac{1}{2}}$. Meanwhile, if $\sigma \leq
\frac{\beta_R}{(Cmn)^\frac{1}{4}}$,
then we have $\beta^2_R  \geq \sigma^2 (Cmn)^{\frac{1}{2}} \geq 2\| \mG \| $.
Combining these two conditions we obtain the first case in the theorem. When
$\sigma \leq \frac{ \beta_1 } { \sqrt{m} }$,  we can see that $\| \mG \| \leq
\frac{\sigma}{2} (Cn)^{\frac{1}{2}} \beta_1$. Moreover, since
$m / n > C ( \beta_1 / \beta_R)^4$, it is implied that
$\sigma \leq \beta^2_R / (Cn)^{\frac{1}{2}} \beta_1 $ and thus
$\beta^2_R \geq 2 \| \mG \|$. This gives us the second case.

On the other hand, if $ m / n \leq C (\beta_1/ \beta_R)^4$, we
require $\sigma \leq \beta^2_R / (Cn)^\frac{1}{2} \beta_1$ to
obtain the last case in the theorem.

\begin{lemma}
\label{lemma:operator-norm-rankR}
Let $\mG$ be constructed as in the proof of Theorem \ref{thm:leftsv_rankR}. If
$m\geq n$, 
there exists an universal constant $\bar{C}$ such that 
\[ \| \mG \| \leq \bar{C} \sigma^2 \left(\sqrt{mn} + \sqrt{n} \beta_1 / \sigma \right), \]
with probability at least $1-3e^{-n}$.
\end{lemma}
\begin{proof}
The proof is an $\epsilon$-net argument.
Let
\[ 
    \lambda = 2\sigma^2 \left(\sqrt{4mn} + 4n +  \sqrt{ R + 8n + 4 \sqrt{Rn} }
        \cdot \beta_1 / \sigma \right).
\]
The goal is to control $ | \vx^\T \mG \vx |$ for all the vectors $\vx$ on the unit Euclidean sphere
$\tS^{n-1}$. In order to do this, we first bound the probability of the tail event $|\vx^\T \mG \vx| > \lambda$,
for any fixed $\vx \in \tS^{n-1}$.
Then we bound the probability that $|\vx^\T \mG \vx| > \lambda $ for all the vectors in a 
$\epsilon$-net $\tN_{\epsilon}$. Finally, we establish the connection
between $\sup_{\vx \in \tN_{\epsilon} } | \vx^\T \mG \vx |$ and $\| \mG \|$.

To bound $\P( |\vx^\T \mG \vx| > \lambda)$ for a fix $\vx \in
 \tS^{n-1}$, we expand $\vx^\T \mG \vx$ as
 \[ 
     \vx^\T\mG \vx=\sigma^2 (\| \vz \|^2 - m) + 2\sigma\sum_{r=1}^R \beta_r
     \gamma_r (\vu_r^\T\vx) , 
 \]
where $\vz = \mE^\T\vx$ and $\gamma_r=\vv_r^\T\vz$. It is easy to see that
$\gamma_r \sim\tN(0, 1)$, 
$\vz \sim \tN(0, \mI)$ and $\|\vz\|^2$ is $\chi^2$ distributed with $m$
degrees of freedom.

Let $\vgamma = [\gamma_1, \ldots, \gamma_R]$ and $\vomega = [ \vu^\T_1 x,
\ldots, \vu^\T_R x]$.
We have
\begin{align*}
    |\vx^\T \mG \vx| \leq & \sigma^2 \big|  \| \vz \|^2 - m \big|  + 2\sigma \bigg| \sum_{r=1}^R  \beta_r \gamma_r  (\vu^\T_r \vx)  \bigg| \\
                     \leq & \sigma^2 \big|  \| \vz \|^2 - m \big| + 2 \sigma
    \sum_{r=1}^R \max_{r \in [R]} | \beta_r | \cdot | \gamma_r | \cdot |\vu^\T_r \vx |  \\
    \leq &  \sigma^2 \big|  \| \vz \|^2 - m \big| + 2 \sigma \beta_1
    \cdot  \| \vgamma \| \cdot \| \vomega \| \\
    \leq &  \sigma^2 \big|  \| \vz \|^2 - m \big| + 2 \sigma \beta_1
    \cdot \| \vgamma \|,
\end{align*}
where we used the Cauchy-Schwarz inequality in the second to last line and the
fact $\| \vomega \| \leq 1$ in the last line. 
Note that $\gamma_1, \ldots, \gamma_R$ are i.i.d standard Gaussian distributed so that
$\| \vgamma \|^2$ is $\chi^2$ distributed with $R$ degrees.

First we bound the deviation of the $\chi^2_m$ term. By the corollary of
 Lemma 1 in \citet{LauMas00}, we have
\begin{align}
\label{ineq:chi2}
    \P(\big| \| \vz \|^2 - m \big| > \lambda_1) \leq  2 e^{-4n},
\end{align}
where $\lambda_1=2(\sqrt{4mn}+4n)$.

Next we bound the $\chi^2_R$ term. Similarly, we have
\begin{align}
\label{ineq:chi2_R}
    \P( \| \vgamma \|^2 - R > \lambda_2) \leq  e^{-4n},
\end{align}
where $\lambda_2 = 2(\sqrt{4Rn}+4n)$.

Combining inequalities \eqref{ineq:chi2} and \eqref{ineq:chi2_R}, we have
\begin{align*}
&\P(|\vx^\T \mG \vx| > \lambda)\\
&\leq\; \P\left( \sigma^2 \big|  \| \vz \|^2 - m \big|  +
2\sigma \beta_1   \| \vgamma \| > \sigma^2\lambda_1 +
2 \sigma \beta_1  \sqrt{ R + \lambda_2} \right)\\
&\leq \; \P\left( \big| \| \vz \|^2 - m \big| > \lambda_1 
    \lor \| \vgamma \| > \sqrt{ R + \lambda_2} \right) \\
    &\leq \; \P\left( \big| \| \vz \|^2 - m \big| > \lambda_1 \right) +
    \P \left( \|\vgamma\| >  \sqrt{ R + \lambda_2 } \right)\\
    &\leq \; 3e^{-4n},
\end{align*}
where the second to last line follows from the union bound.

Furthermore, using Lemma 5.2 and 5.4 of
\citet{Ver10}, for any $\epsilon \in [0, 1)$, it holds that 
\[
    |\tN_\epsilon| \leq (1+ 2/\epsilon)^n, 
\]
and 
\begin{equation*}
    \| \mG \| \leq (1-2\epsilon)^{-1} \sup_{\vx \in \tN_\epsilon} | \vx^\T \mG
    \vx |.
\end{equation*}
Taking the union bound over all the vectors in $\tN_{1/4}$, we obtain
\[
    \P( \| \mG \| \leq 2 \lambda ) \leq 
\P\left(\sup_{\vx \in \tN_{1/4} } |\vx^\T \mG \vx| > \lambda\right) \leq
| \tN_{1/4} | 3e^{-4n} < 3e^{-n}.
\]
Finally, the statement is obtained by noticing that 
$n \leq m$ and $ R \leq n $.
\end{proof}

\end{document}